\documentclass[notitlepage]{article}

\usepackage{arxiv}
 \usepackage{algorithm}
 \usepackage{algorithmic}
 \usepackage{amsmath}
 \usepackage{amsthm}
 \usepackage{subfig}
 \usepackage{hyperref}
 \usepackage[nameinlink,capitalize,noabbrev]{cleveref}
 \usepackage{bm}
 \usepackage{amssymb}
 \usepackage{mathtools}
 \usepackage{enumitem}
 \usepackage{xspace}
 \usepackage{xcolor}
 \hypersetup{
 	colorlinks,
 	linkcolor=black,
 	urlcolor=pink,
 	citecolor=blue}
 \usepackage{booktabs}
 \usepackage{multirow}
 \usepackage{graphicx}
 \usepackage{tablefootnote}
 \usepackage{tikz}
 \usepackage{float}
 \usepackage{caption}
\usepackage[round]{natbib}

\newtheorem{theorem}{Theorem}
\newtheorem{lemma}{Lemma}[section]

\newtheorem{assumption}{Assumption}

\newtheorem*{assumption*}{Assumption}

\usepackage{circuitikz} 
\usetikzlibrary{shapes.geometric, positioning, calc, arrows, automata}

\DeclareMathOperator*{\argmax}{argmax}

\newcommand{\regret}{\mathrm{Regret}}

\usepackage{lastpage}

\begin{document}

\title{Learning Infinite-Horizon Average-Reward Linear Mixture MDPs of Bounded Span}

\author{%
  Woojin Chae \\
  KAIST \\
  \texttt{woojeeny02@kaist.ac.kr} \\
   \And
   Kihyuk Hong \\
  University of Michigan\\
  \texttt{kihyukh@umich.edu} \\
  \And
  Yufan Zhang \\
  University of Michigan \\
  \texttt{yufanzh@umich.edu} \\
 \AND
  Ambuj Tewari \\
  University of Michigan \\
  \texttt{tewaria@umich.edu} \\
  \And
  Dabeen Lee \\
  KAIST\\
  \texttt{dabeenl@kaist.ac.kr} \\
}

\maketitle

\begin{abstract}
This paper proposes a computationally tractable algorithm for learning infinite-horizon average-reward linear mixture Markov decision processes (MDPs) under the Bellman optimality condition. Our algorithm for linear mixture MDPs achieves a nearly minimax optimal regret upper bound of $\widetilde{\mathcal{O}}(d\sqrt{\mathrm{sp}(v^*)T})$ over $T$ time steps where $\mathrm{sp}(v^*)$ is the span of the optimal bias function $v^*$ and $d$ is the dimension of the feature mapping. Our algorithm applies the recently developed technique of running value iteration on a discounted-reward MDP approximation with clipping by the span. We prove that the value iteration procedure, even with the clipping operation, converges. Moreover, we show that the associated variance term due to random transitions can be bounded even under clipping. Combined with the weighted ridge regression-based parameter estimation scheme, this leads to the nearly minimax optimal regret guarantee.
\end{abstract}

\section{Introduction}\label{sec:introduction}

Reinforcement learning (RL) with function approximation has achieved remarkable success in a wide range of areas, including video games~\citep{atari-mnih}, Go~\citep{go-silver}, robotics~\citep{robotics}, and autonomous driving~\citep{driving}. Such empirical progress has stimulated endeavors to expand our theoretical understanding of RL with function approximation. 

As a first step toward establishing theoretical foundations, linear function approximation frameworks have received significant attention. The works on linear function approximation can be categorized based on how linearity is assumed on the structure of the underlying Markov decision process (MDP). There are largely four settings: MDPs with a low Bellman rank~\citep{jiang17}, linear MDPs~\citep{yang-wang-2019,jin-linear-2020}, linear mixture MDPs~\citep{pmlr-v120-jia20a,pmlr-v119-ayoub20a,pmlr-v139-zhou21a}, and MDPs with a low inherent Bellman error~\citep{pmlr-v119-zanette20a}.

Among them, notable progress has been made for linear mixture MDPs where the underlying transition kernel and the reward function are assumed to be parameterized as a linear function of some given feature mappings over state-action pairs or state-action-state triplets. For the finite-horizon setting, \cite{zhou-mixture-finite-optimal} developed a nearly minimax optimal algorithm. For the infinite-horizon discounted-reward case, \cite{pmlr-v139-zhou21a} established a regret lower bound. Shortly after this, \cite{zhou2021nearlyminimaxoptimalreinforcement} announced an algorithm with a regret upper bound matching the lower bound up to logarithmic factors. For the infinite-horizon average-reward regime, \cite{yuewu2022} showed a regret lower bound of $\Omega(d\sqrt{DT})$ from a communicating MDP instance with diameter $D$ where $d$ is the dimension of the feature map and $T$ is the horizon. Moreover, they designed an algorithm that achieves a regret upper bound of $\widetilde{\mathcal{O}}(d\sqrt{DT})$, establishing near minimax optimality for the class of communicating MDPs.

For the infinite-horizon average-reward setting, however, the class of communicating MDPs is perhaps not the most general set of MDPs for which a learning algorithm can guarantee a sublinear regret~\citep{tewari12}. Although the diameter captures the number of steps needed to recover from a bad state to a good state, the actual regret incurred while recovering is better represented by the span $\mathrm{sp}(v^*)$ of the optimal bias function $v^*$~\citep{pmlr-v80-fruit18a}. While the diameter is an upper bound on the span, it can be arbitrarily larger than the span, and in fact, a weakly communicating MDP can have a finite span but an infinite diameter~\citep{tewari12}.

The recent framework of \cite{he2024sampleefficient}, \texttt{LOOP}, can be applied to infinite-horizon average-reward linear mixture MDPs of bounded span. Although \texttt{LOOP} guarantees a sublinear regret upper bound that depends on the span, \texttt{LOOP} is hardly practical as it relies on solving a complex constrained optimization problem. This motivates the following question. 
\begin{center}
\emph{Does there exist a computationally efficient, nearly minimax optimal algorithm for learning infinite-horizon average-reward linear mixture MDPs of bounded span?}
\end{center}
This paper answers the question affirmatively. Let us summarize our contributions in \Cref{results} and as follows.
\begin{table*}[h!]
\caption{{Summary of Our Results on Regret Upper and Lower Bounds for Learning Linear Mixture MDPs}}\label{results}
\begin{center}
\begin{tabular}{c|c|c}
\toprule
{\bf \small Setting} & {\bf \small Regret Upper Bound}   & {\bf\small Regret Lower Bound}\\
\midrule
{\small Communicating (bounded diameter)} & {\small $\widetilde{\mathcal{O}}\left(d\sqrt{DT}\right)$ \hfill \citep{yuewu2022}} & {\small $\Omega\left(d\sqrt{DT}\right)$\hfill \citep{yuewu2022} }\\
\midrule
{\small Bellman optimality (bounded span)} & {\small $\widetilde{\mathcal{O}}\left(d \sqrt{\mathrm{sp}(v^*)T}\right)$\hfill \quad (\Cref{thm:ub}) }& {\small $\Omega\left(d \sqrt{\mathrm{sp}(v^*)T}\right)$\hfill \quad (\Cref{thm:lb})}\\
\bottomrule
\end{tabular}
\end{center}
\end{table*}
\begin{itemize}
   
    \item We propose a computationally efficient algorithm, upper-confidence linear kernel reinforcement learning with clipping (\texttt{UCLK-C}; \Cref{alg:UCLK-C}) that achieves a regret upper bound of $\widetilde{\mathcal{O}}(d\sqrt{\mathrm{sp}(v^*)T})$. 
     \item We deduce a regret lower bound of $\Omega(d\sqrt{\mathrm{sp}(v^*)T})$ by refining the regret lower bound analysis of~\cite{yuewu2022}. This shows that \texttt{UCLK-C} is nearly minimax optimal. 
  \item \texttt{UCLK-C} runs with a novel value iteration scheme by applying the clipping operation within discounted extended value iteration. The clipping operation due to~\cite{hong2024provablyefficientreinforcementlearning} is for controlling the span of intermediate value functions, which is crucial to provide a bounded regret for MDPs of bounded span. The clipping operation is much simpler to implement than constrained optimization-based frameworks to control the span. To run discounted extended value iteration, we approximate a given average-reward MDP by a discounted-reward MDP, as in \texttt{UCLK} due to~\cite{pmlr-v139-zhou21a}. 
  \item We prove that for linear mixture MDPs, the discounted extended value iteration converges even with clipping. Moreover, we show that the associated variance term due to random transitions can be bounded even under the clipping operation. Combined with the variance-aware weighted ridge regression-based parameter estimation scheme due to~\cite{yuewu2022}, we deduce our nearly minimax optimal regret upper bound.
\end{itemize}

The idea of approximating an average-reward MDP by a discounted-reward MDP has been adopted for the tabular case~\citep{pmlr-v119-wei20c,pmlr-v195-zhang23b} and used for learning linear MDPs~\citep{hong2024provablyefficientreinforcementlearning}. Clipping an optimistic value function estimator to control its size is already a common practice when designing an algorithm for finite-horizon and infinite-horizon discounted-reward MDPs. However, the clipping operation in our algorithm sets the threshold in a different way to control the span of value functions, not their sizes, and it was first introduced by \cite{hong2024provablyefficientreinforcementlearning}. For linear MDPs, due to the clipping operation, convergence of value iteration is not guaranteed. In contrast, for linear mixture MDPs, we establish convergence of value iteration even with clipping.

\section{Related Work}

\paragraph{Reinforcement Learning with Linear Function Approximation} 

Recently, there has been remarkable progress in reinforcement learning frameworks with linear function approximation~\citep{jiang17,yang-wang-2019,yang-wang-2020,jin-linear-2020,wang2021optimism,modi-linear-mixture,NEURIPS2018_5f0f5e5f,bilinear,pmlr-v99-sun19a,pmlr-v108-zanette20a,pmlr-v119-zanette20a,pmlr-v119-cai20d,pmlr-v120-jia20a,pmlr-v119-ayoub20a,pmlr-v132-weisz21a,pmlr-v139-zhou21a,zhou-mixture-finite-optimal,pmlr-v139-he21c,NEURIPS2022_ebba182c,Hu2022,He2023,Agarwal23,hong2024provablyefficientreinforcementlearning}. These works develop frameworks for MDP classes with certain linear structures. Among them, the most relevant to this paper are linear and linear mixture MDPs. Linear and linear mixture MDPs assume that the transition probability and the reward function are linear in some feature mappings over state-action pairs or state-action-state triplets. 
Although the two classes are closely related, one cannot be covered by the other~\citep{pmlr-v139-zhou21a}. %
For learning infinite-horizon average-reward linear MDPs, \cite{pmlr-v130-wei21d} developed several algorithms, including \texttt{FOPO}. \texttt{FOPO} achieves the best-known regret upper bound, but it needs to solve a fixed-point equation at each iteration, making the algorithm intractable. Recently,~\cite{hong2024provablyefficientreinforcementlearning} proposed a provably efficient algorithm that achieves the best-known regret upper bound for the setting. For learning infinite-horizon average-reward linear mixture MDPs, \cite{yuewu2022} developed an algorithm that is shown to be minimax optimal for the communicating case. \cite{he2024sampleefficient} proposed an algorithm for RL with general function approximation, \texttt{LOOP}, incorporating linear and linear mixture MDPs as subclasses. 

\paragraph{Infinite-Horizon Average-Reward Reinforcement Learning} The seminal work by \cite{NIPS2008_e4a6222c} pioneered algorithmic frameworks for model-based online learning of MDPs. They proved a regret lower bound of $\Omega(\sqrt{DSAT})$ where $S$ is the number of states and $A$ is the number of actions. Then they provided UCRL2 that runs with constructing some optimistic sets for estimating the transition probability and guarantees a regret bound of $\widetilde{\mathcal{O}}(DS\sqrt{AT})$. \cite{tewari12} considered the class of MDPs of bounded span, for which they proposed an algorithm that achieves a regret upper bound of $\widetilde{\mathcal{O}}(\mathrm{sp}(v^*)S\sqrt{AT})$. Since then, there has been a long line of work toward closing the gap between regret upper and lower bounds~\citep{Filippi10,pmlr-v83-talebi18a,pmlr-v80-fruit18a,fruit2020improvedanalysisucrl2empirical,pmlr-v119-bourel20a,NEURIPS2019_9e984c10,NIPS2017_3621f145,NIPS2017_51ef186e,pmlr-v97-lazic19a,pmlr-v130-wei21d,pmlr-v195-zhang23b,boone2024achievingtractableminimaxoptimal}. In particular,~\cite{pmlr-v80-fruit18a} and~\cite{NEURIPS2019_9e984c10} refined the regret lower bound to $\Omega(\sqrt{\mathrm{sp}(v^*)SAT})$. The first result with a regret upper bound matching the lower bound is due to \cite{NEURIPS2019_9e984c10}, but their algorithm is not tractable. Recently,~\cite{boone2024achievingtractableminimaxoptimal} developed a tractable algorithm that guarantees a regret bound of $\widetilde{\mathcal{O}}(\sqrt{\mathrm{sp}(v^*)SAT})$. 

\section{Preliminaries} \label{sec:preliminaries}

\paragraph{Notations} Given a vector $x\in\mathbb{R}^d$ and a positive semidefinite matrix $A\in\mathbb{R}^{d\times d}$, $\|x\|_2$ is the $\ell_2$-norm, $\|x\|_A = \sqrt{x^{\top}Ax}$, and $\|A\|_2$ is the spectral norm. %
For any positive integers $m,n$ with $m<n$, $[n]$ and $[m:n]$ denote $\{1,\ldots,n\}$ and $\{m,\ldots,n\}$, respectively.

\paragraph{Infinite-Horizon Average-Reward MDP}
We consider an MDP given by $M=(\mathcal{S}, \mathcal{A}, \mathbb{P}, r)$ where $\mathcal{S}$ is the state space, $\mathcal{A}$ is the action space, $\mathbb{P}(\cdot \mid s,a)$ specifies the transition probability function for state $s$ with taking action $a$, and $r(s,a)\in[0,1]$ is the reward from action $a$ at state $s$. A (stochastic) stationary policy is given as a mapping $\pi : \mathcal{S} \rightarrow \Delta(\mathcal{A})$ where $\Delta(\mathcal{A})$ is the set of probability measures on $\mathcal{A}$, and we use notation $\pi(a\mid s)$ for the probability of taking action $a$ at state $s$ under policy $\pi$. When $\pi$ is a deterministic policy, we write that $a=\pi(s)$ with abuse of notation where $a$ is the action with $\pi(a\mid s)=1$. At each time step $t$, an algorithm takes action $a_t$ at given state $s_t$, after which it observes the next state $s_{t+1}$ drawn from distribution $\mathbb{P}(\cdot \mid s_t,a_t)$. Then the cumulative reward over $T$ steps is $\sum_{t=1}^T r(s_t,a_t)$. Then the (long-term) average reward is given by $\liminf_{T\to\infty}\mathbb{E}[\sum_{t=1}^T r(s_t,a_t)]/T$, which can be maximized by a deterministic stationary policy~\cite[See][]{puterman2014markov}. We denote by $J^{\pi}(s)=\liminf_{T\to\infty}\mathbb{E}[\sum_{t=1}^T r(s_t,a_t)\mid s_1=s]/T$ the average reward of a stationary policy $\pi$ starting from initial state $s$.

In this paper, we focus on the class of MDPs satisfying the following form of Bellman optimality condition. We assume that there exist $J^*\in\mathbb{R}$, $v^*:\mathcal{S}\to\mathbb{R}$, and $q^*:\mathcal{S}\times\mathcal{A}\to\mathbb{R}$ such that for all $(s,a)\in\mathcal{S}\times \mathcal{A}$,
\begin{equation}\label{bellman-average}
\begin{aligned}
J^* + q^*(s,a)&= r(s,a) + \mathbb{E}_{s'\sim \mathbb{P}(\cdot\mid s,a)}\left[v^*(s')\right],\quad v^*(s)=\max_{a\in\mathcal{A}}q^*(s,a).
\end{aligned}
\end{equation}
Under the Bellman optimality condition, the optimal average reward $J^*(s):=\max_{\pi}J^{\pi}(s)$ is invariant with the initial state $s$, and $J^*(s)=J^*$ for any $s\in\mathcal{S}$ \citep{tewari12}. Moreover, the class of weakly communicating MDPs satisfies the condition \citep[See][]{puterman2014markov}. There indeed exist other general classes of MDPs with which the condition holds \citep[Section 3.3]{10.5555/516568}. For any function $h:\mathcal{S}\to\mathbb{R}$, we define its span as $\mathrm{sp}(h):=\max_{s\in\mathcal{S}}h(s)-\min_{s\in\mathcal{S}}h(s)$. Then following the literature on infinite-horizon average-reward RL~\citep{NIPS2008_e4a6222c}, we consider the regret function $\regret(T)= T\cdot J^* - \sum_{t=1}^T r(s_t,a_t)$ to assess the performance of an algorithm.

\paragraph{Discounted-Reward MDP} We also consider the discounted cumulative reward of a stationary policy $\pi$ given by $V^{\pi}(s)=\mathbb{E}[\sum_{t=1}^{\infty}\gamma^{t-1} r(s_t,a_t)\mid s_1=s]$ where $s$ is the initial state and $\gamma\in(0,1)$ is a discount factor. Similarly, we consider $Q^{\pi}(s,a)=\mathbb{E}[\sum_{t=1}^{\infty}\gamma^{t-1} r(s_t,a_t)\mid(s_1,a_1)=(s,a)]$. Then we define the optimal value function $V^*$ and the optimal action-value function $Q^*$ as $V^*(s)=\max_{\pi} V^{\pi}(s)$ and $Q^*(s,a)=\max_{\pi} Q^{\pi}(s,a)$ for $(s,a)\in\mathcal{S}\times\mathcal{A}$. It is known that there exists a deterministic stationary policy that gives rise to $V^*$ and $Q^*$~\citep[See][]{puterman2014markov,agarwal21}. Moreover, $V^*$ and $Q^*$ satisfy the following Bellman optimality equation.
\begin{equation}\label{bellman-discounted}
\begin{aligned}
Q^*(s,a) &= r(s,a) + \gamma\mathbb{E}_{s'\sim \mathbb{P}(\cdot\mid s,a)}\left[V^*(s')\right],\quad 
V^*(s)=\max_{a\in\mathcal{A}}Q^*(s,a).
\end{aligned}
\end{equation}
Our approach is to approximate an average-reward MDP by a discounted-reward MDP. In fact, as the discount factor gets close to 1, the discounted cumulative reward converges to the average reward for a stationary policy~\citep[See][]{puterman2014markov}.

\paragraph{Linear Mixture MDPs} In this work, we focus on linear mixture MDPs, defined formally as follows.
\begin{assumption}\label{assumption:mixture}
    {\em \citep[{Linear Mixture MDP},][]{zhou-mixture-finite-optimal,yuewu2022}} There is a known feature map $\phi:\mathcal{S}\times\mathcal{A}\times\mathcal{S}\to\mathbb{R}^d$ such that for any $(s,a,s')\in\mathcal{S}\times\mathcal{A}\times\mathcal{S}$,
    $$\mathbb{P}(s'\mid s,a)=\langle\phi(s,a,s'),{\theta^*}\rangle$$
    where ${\theta^*}\in\mathbb{R}^d$ is an unknown vector.
\end{assumption}
We assume that the reward function is deterministic and known to the decision-maker, but our results easily extend to the setting where the reward function is given by 
$r(s,a)=\langle\varphi(s,a),{\theta^*}\rangle$
for some feature map $\varphi:\mathcal{S}\times\mathcal{A}\rightarrow\mathbb{R}^d$. For any function $F:\mathcal{S}\to\mathbb{R}$, we use the following shorthand notations.
\begin{align*}
[\mathbb{P}F](s,a)&=\mathbb{E}_{s'\sim\mathbb{P}(\cdot\mid s,a)}[F(s')],\\
[\mathbb{V}F](s,a)&=[\mathbb{P}F^2](s,a)-([\mathbb{P}F](s,a))^2
\end{align*}
Defining $\phi_F(s,a)=\int_{s'}\phi(s,a,s')F(s')ds'$, we have
\begin{align*}
\left\langle\phi_F(s,a),\theta^*\right\rangle&= \int\left\langle\phi(s,a,s'),\theta^*\right\rangle F(s')ds'=[\mathbb{P}F](s,a).
\end{align*}
Therefore, we also have 
$$[\mathbb{V}F](s,a)=\left\langle\phi_{F^2}(s,a),\theta^*\right\rangle-\left\langle\phi_F(s,a),\theta^*\right\rangle^2.$$
Following~\cite{yuewu2022}, we assume that the scales of the parameters are bounded as follows.
\begin{assumption}\label{assumption:scale}
$\theta^*$ satisfies $\|{\theta^*}\|_2\leq B_\theta$ for some $B_\theta\in \mathbb{R}$. Moreover, for any $H>0$, $F:\mathcal{S}\to[0,H]$, and $(s,a)\in\mathcal{S}\times\mathcal{A}$, it holds that $\|\phi_F(s,a)\|_2 \leq H$. %
\end{assumption}

\section{The Proposed Algorithm}

In this section, we present our algorithm, \texttt{UCLK-C}, described in \Cref{alg:UCLK-C}.
\begin{algorithm*}[tb] 
\renewcommand\thealgorithm{1}
\caption{Upper-Confidence Linear Kernel Reinforcement Learning with Clipping (\texttt{UCLK-C})}
\label{alg:UCLK-C}
\begin{algorithmic}[1]
\STATE\textbf{Input:} 
upper bound $H$ of $2\cdot \mathrm{sp}(v^*)$, feature map $\phi:\mathcal{S} \times \mathcal{A} \times \mathcal{S} \rightarrow \mathbb{R}^d$, 
    confidence level $\delta \in (0,1)$, discount factor $\gamma\in[0,1)$, number of rounds $N$, and parameters $\lambda, B_\theta$
\STATE\textbf{Initialize:} 
  $t \leftarrow 1$,  $\widehat\Sigma_1,\widetilde\Sigma_1 \leftarrow \lambda {I_d}$, $\widehat b_1, \widetilde b_1,\widehat{{\theta}}_1,\widetilde{{\theta}}_1\leftarrow 0$,
    and observe the initial state $s_1 \in \mathcal{S}$ 
\FOR{episodes $k=1,2,\ldots,$}
\STATE Set $t_k = t$ and $\mathcal{C}_k=\widehat{\mathcal{C}}_{t_k}$ given in~\eqref{eq:confidence-set} \label{alg:confidence-ellipsoid}
\STATE Initialize $Q^{(0)}(\cdot,\cdot) \leftarrow (1-\gamma)^{-1}$ and $V^{(0)}(\cdot) \leftarrow (1-\gamma)^{-1}$
\FOR {rounds $n=1,2,\ldots,N$}\label{alg:value-iteration-start}
\STATE Set $Q^{(n)}(\cdot,\cdot)\leftarrow r(\cdot,\cdot)+\gamma \cdot\max_{\theta\in\mathcal{C}_k}\langle \phi_{V^{(n-1)}}(\cdot,\cdot),\theta\rangle$ \label{alg:value-iteration}
\STATE Set $\widetilde V^{(n)}(\cdot)\leftarrow \max_{a\in\mathcal{A}} Q^{(n)}(\cdot,a)$ \label{alg:intermediate-value}
\STATE Set $V^{(n)}(\cdot)\leftarrow \min\{\widetilde V^{(n)}(\cdot), \min_{s'\in\mathcal{S}}\widetilde V^{(n)}(s')+H\}$ \label{alg:clipping}
\ENDFOR\label{alg:value-iteration-end}
\STATE Set $Q_k(\cdot,\cdot)\leftarrow Q^{(N)}(\cdot,\cdot)$ and $V_k(\cdot)\leftarrow V^{(N)}(\cdot)$\label{alg:value-functions}
\STATE Set $W_k(\cdot) \leftarrow V_k(\cdot) -\min_{s'\in\mathcal{S}}V_k(s')$ \label{alg:recentering}
\STATE Take a deterministic policy $\pi_k$ given by $\pi_{k}(\cdot)\in \argmax_{a\in\mathcal{A}} Q_k(\cdot ,a)$ \label{planning-phase-end}
    \WHILE{$\det(\widehat{\Sigma}_t) \leq 2\det(\widehat{\Sigma}_{t_k})$}\label{execution-start}
        \STATE Take action $a_t \leftarrow \pi_k(s_t)$, receive reward $r(s_t,a_t)$ and next state $s_{t+1}\sim\mathbb{P}(\cdot|s_t, a_t)$
        \STATE Set $\bar{\sigma}_t \leftarrow\sqrt{\max\left\{H^2/d, [\bar{\mathbb{V}}_tW_k](s_t, a_t) + E_t\right\}}$ where $[\bar{\mathbb{V}}_t W_k](s_t, a_t)$ and $E_t$ are given as in (\ref{algo:variance_estimator}) and (\ref{algo:var_error})\label{alg:variance}
        \STATE Set $\widehat{\Sigma}_{t+1} \leftarrow \widehat{\Sigma}_{t} + \bar{\sigma}_t^{-2}\phi_{W_k}(s_t, a_t)\phi_{W_k}(s_t, a_t)^\top$ and $\widehat{b}_{t+1} \leftarrow \widehat{b}_t + \bar{\sigma}_t^{-2} W_k(s_{t+1}) \phi_{W_k}(s_t, a_t)$\label{alg:parameter1}
        \STATE Set $\widetilde{\Sigma}_{t+1} \leftarrow \widetilde{\Sigma}_t + \phi_{W_k^2}(s_t, a_t)\phi_{W_k^2}(s_t, a_t)^\top$ and $\widetilde{b}_{t+1} \leftarrow \widetilde{b}_t + W_k^2(s_{t+1})\phi_{W_k^2}(s_t, a_t)$\label{alg:parameter2}
        \STATE Set $\widehat{\theta}_{t+1} \leftarrow \widehat{\Sigma}_{t+1}^{-1}\widehat{b}_{t+1}$ and  $\widetilde{\theta}_{t+1} \leftarrow \widetilde{\Sigma}_{t+1}^{-1}\widetilde{b}_{t+1}$\label{alg:parameter3}
        \STATE Set $t \leftarrow t+1$
    \ENDWHILE\label{execution-end}
\ENDFOR
\end{algorithmic}
\end{algorithm*}
As common in algorithms for learning infinite-horizon average-reward MDPs such as \texttt{UCRL2}~\citep{NIPS2008_e4a6222c} and \texttt{UCRL2-VTR}~\citep{yuewu2022}, \texttt{UCLK-C} also proceeds with episodes. Following \texttt{UCRL2-VTR}, when to start the next episode is determined based on the Gram matrix (line~\ref{execution-start}). Each episode of \texttt{UCLK-C} consists of two phases, the planning phase (lines \ref{alg:confidence-ellipsoid}--\ref{planning-phase-end}) and the execution phase (lines \ref{execution-start}--\ref{execution-end}). During the planning phase, we run extended value iteration for a discounted MDP with estimated parameters. Then, based on value functions deduced from the planning phase, we take and execute a greedy deterministic (non-stationary) policy for the execution phase. What follows provides a more detailed discussion of the components of \texttt{UCLK-C}.

\paragraph{Discounted Value Iteration}

As in~\cite{hong2024provablyefficientreinforcementlearning}, we apply value iteration on a discounted-reward approximation of the underlying MDP. For each episode $k$, we take a confidence ellipsoid $\mathcal{C}_k$ (line~\ref{alg:confidence-ellipsoid}) over which we run extended value iteration (line~\ref{alg:value-iteration}). We make sure that any $\theta\in \mathcal{C}_k$ induces a probability distribution, i.e. $\mathcal{C}_k\subseteq \mathcal{B}$ where
$$\mathcal{B}=\left\{\theta\in\mathbb{R}^d:\ 
\begin{aligned}
&\left\langle \int \phi(s,a,s') ds',\theta\right\rangle =1,\ \langle \phi(s,a,s'),\theta\rangle\geq 0
\end{aligned}\ \forall (s,a,s')\right\}.$$
As a result, we get that $Q^{(n)}(s,a)\leq (1-\gamma)^{-1}$ for any $(s,a)\in\mathcal{S}\times\mathcal{A}$ and $n$. Note that we apply multiple rounds of value iteration. We will show that even with the clipping operation (line~\ref{alg:clipping}), the value iteration procedure stated in lines~\ref{alg:value-iteration-start}--\ref{alg:value-iteration-end} converges.

\paragraph{Clipping Operation}
In each round of value iteration, \texttt{UCLK-C} applies the clipping operation stated in line~\ref{alg:clipping}. Note that the value function $\widetilde{V}^{(n)}$ from line~\ref{alg:intermediate-value} does not necessarily have a bounded span. After the clipping operation, it is clear that the span of $V^{(n)}$ from line~\ref{alg:clipping} becomes bounded above as $\mathrm{sp}(V^{(n)})\leq H$. As a result, the re-centering step in line~\ref{alg:recentering} guarantees that $W_{k}(s)\in[0,H]$ for any $s\in\mathcal{S}$, which is crucial to parameterize estimation errors as a function of $H$, not $(1-\gamma)^{-1}$ which is set as large as $O(\sqrt{T})$. Note that the Gram matrix update steps (lines~\ref{alg:parameter1}--\ref{alg:parameter2}) are  with respect to $W_k$, not $V_{k}$. We choose any upper bound $H$ on $2\cdot \mathrm{sp}(v^*)$ where $v^*$ is the optimal bias function from the Bellman optimality condition~(\ref{bellman-average}).

\paragraph{Variance-Aware Ridge Regression-Based Parameter Estimation} We closely follow the Bernstein inequality-based estimation scheme of \cite{yuewu2022}.  The idea is to build confidence ellipsoids for $\theta^*$ based on a Bernstein-type concentration inequality for linear bandits. To be more precise, we apply the following lemma for vector-valued martingales.
\begin{lemma}{\em \citep[Theorem 4.1,][]{zhou-mixture-finite-optimal}}\label{lem:bernstein}
Let $\{\mathcal{G}_t\}_{t=1}^{\infty}$ be a filtration, $\{{x}_t, \eta_t\}_{t\geq1}$ a stochastic process such that ${x}_t\in\mathbb{R}^d$ is $\mathcal{G}_t$-measurable while $\eta_t\in\mathbb{R}$ is  $\mathcal{G}_{t+1}$-measurable. For $t \geq 1$, let $y_t = \langle {x}_t,{\mu}^* \rangle + \eta_t$ where $|\eta_t| \leq R$, $\mathbb{E}[\eta_t\mid\mathcal{G}_t]=0$, $\mathbb{E}[\eta_t^2\mid\mathcal{G}_t]\leq\sigma^2$, and $\|{x}_t\|_2 \leq L$ for some fixed $R$, $L$, $\sigma$, $\lambda>0$ and $\mu^* \in \mathbb{R}^d$.
Then, for any $0 < \delta < 1$, it holds with probability at least $1-\delta$ that for every $t\geq 1$,
$$\left\|\mu_t-\mu^*\right\|_{Z_t} \leq \beta_t+\sqrt{\lambda}\|\mu^*\|_2$$
where $\beta_t=8\sigma\sqrt{d\log(1+tL^2/(d\lambda))\log(4t^2/\delta)}+4R\log(4t^2/\delta)$, $\mu_t={Z}_t^{-1}{b}_t$, $Z_t=\lambda{I}_d+\sum_{i=1}^{t}{x}_i{x}_i^\top$, and ${b}_t=\sum_{i=1}^{t}y_i{x}_i$.
\end{lemma}
Note that $\widehat \theta_{t+1}$ in line~\ref{alg:parameter3} corresponds to $\mu_t$ in \Cref{lem:bernstein} by setting $(x_j,y_j)= (\bar\sigma_j^{-1}\phi_{W_i}(s_j,a_j), \bar\sigma_j^{-1}W_i(s_{j+1}))$ for $j\in[t_i:t_{i+1}-1]$ and $i\in[k]$. Then $\widehat \theta_{t_k}$ is the solution of the following weighted ridge regression.
$$\min_{\theta\in\mathbb{R}^d}   \lambda \|\theta\|_2^2+\sum_{i=1}^{k-1}\sum_{j=t_{i}}^{t_{i+1}-1}\frac{\left( W_{i}(s_{j+1})-\langle \phi_{W_i}(s_j,a_j),\theta\rangle\right)^2}{\bar\sigma_j^2}.$$ Here, $\bar \sigma_j^2$  is an estimator of the (conditional) variance of $W_i(s_{j+1})$,  given by $[\mathbb{V}W_i](s_j,a_j)$. 

We will choose the value of $\bar \sigma_t$ (line~\ref{alg:variance}) for $t\in[t_k,t_{k+1}-1]$ so that $|\eta_t|\leq \sqrt{d}$, $\mathbb{E}[\eta_t^2\mid \mathcal{G}_t]\leq 1$, and $|x_t|\leq \sqrt{d}$, where $\eta_t=y_t- \langle x_t,\theta^*\rangle$ and $\mathcal{G}_t=\sigma(s_1,\ldots,s_t)$ is the $\sigma$-algebra generated by $s_1,\ldots, s_t$. Then, based on \Cref{lem:bernstein}, we may construct
\begin{equation}\label{eq:confidence-set}
\widehat{\mathcal{C}}_{t}=\left\{\theta\in\mathcal{B}: \|\theta-\widehat{\theta}_{t}\|_{\widehat{\Sigma}_{t}} \leq \widehat{\beta}_{t} \right\}
\end{equation}
where we set
\begin{align*}
    \widehat{\beta}_{t} &= 8\sqrt{d\log(1+t/\lambda)\log(4t^2/\delta)}+4\sqrt{d}\log(4t^2/\delta)+\sqrt{\lambda}B_\theta.
\end{align*}
We make sure that $\bar\sigma_t$ is strictly positive. Moreover, to guarantee $\mathbb{E}[\eta_t^2\mid \mathcal{G}_t]\leq 1$, we take the right quantity for $\bar\sigma_t$ so that $\bar \sigma_t^2$ is an upper bound on the variance term given by $$[\mathbb{V}W_k](s_t,a_t)=\langle \phi_{W_k^2}(s_t,a_t),\theta^*\rangle - \langle \phi_{W_k}(s_t,a_t),\theta^*\rangle^2.$$
To estimate this, we first take $[\bar{\mathbb{V}}_tW_k](s_t,a_t)$ given by
\begin{equation}\label{algo:variance_estimator}
\begin{aligned}
[\bar{\mathbb{V}}_tW_k](s_t,a_t)
&=\left[\langle \phi_{W_k^2}(s_t,a_t),\widetilde\theta_t\rangle\right]_{[0,H^2]}- \left[\langle \phi_{W_k}(s_t,a_t),\widehat \theta_t\rangle\right]_{[0,H]}^2
\end{aligned}
\end{equation}
where $[x]_{[a,b]}$ denotes the projection of $x$ onto the interval $[a,b]$ and $\widetilde\theta_{t}$ (line~\ref{alg:parameter3}) corresponds to $\mu_{t-1}$ in \Cref{lem:bernstein} by setting $(x_j,y_j)= (\phi_{W_i^2}(s_j,a_j), W_i^2(s_{j+1}))$ for $j\in[t_i:t_{i+1}-1]$ and $i\in[k]$. Note that $\widetilde\theta_{t}$ is the solution of the (unweighted) ridge regression problem with contexts $\phi_{W_i^2}(s_j,a_j)$ and targets $W_i^2(s_{j+1})$ for $j\in[t_i:t_{i+1}-1]$ and $i\in[k]$.

The last ingredient is to take an upper bound $E_t$ on the error term $|[{\mathbb{V}}W_k](s_t,a_t)-[\bar{\mathbb{V}}_tW_k](s_t,a_t)|$. We take 
\begin{equation}\label{algo:var_error}
\begin{aligned}
E_t&=  \min\left\{H^2, 2H\check{\beta}_{t}\|\phi_{W_k}(s_t,a_t)\|_{\widehat{\Sigma}_{t}^{-1}}\right\} + \min\left\{H^2, \widetilde{\beta}_{t} \|\phi_{W_k^2}(s_t,a_t)\|_{\widetilde{\Sigma}_t^{-1}} \right\}
\end{aligned}
\end{equation}
where we set
\begin{align*}
\check{\beta}_t &= 8d\sqrt{\log(1+t/\lambda)\log(4t^2/\delta)} + 4\sqrt{d}\log(4t^2/\delta) + \sqrt{\lambda}B_\theta,\\
    \widetilde{\beta}_t &= 8H^2\sqrt{d \log(1 + tH^2/(d\lambda))\log(4t^2/\delta)} + 4H^2\log(4t^2/\delta) + \sqrt{\lambda}B_\theta.
\end{align*}
Here, we may observe from \Cref{lem:bernstein} that $\|\theta^*-\widehat \theta_t\|_{\widehat\Sigma_t}\leq \check \beta_t$ and $\|\theta^*-\widetilde \theta_t\|_{\widehat\Sigma_t}\leq \widetilde \beta_t$ with high probability, based on which we prove that $E_t$ provides an upper bound on the error term.

Finally, we set the value of $\bar \sigma_t$ as
$$\bar\sigma_t=\sqrt{\max\left\{H^2/d, [\bar{\mathbb{V}}_tW_k](s_t, a_t) + E_t\right\}}.$$
Note that $\bar\sigma_t$ is well-defined because the term inside the square root is strictly positive. Moreover, $\bar\sigma_t^2$ is greater than or equal to $[\bar{\mathbb{V}}_tW_k](s_t, a_t) + E_t$ which is an upper bound on the variance term $[{\mathbb{V}}W_k](s_t, a_t)$. To formalize this, we prove the following lemma.
\begin{lemma}\label{lem:confidence-ellipsoid}
With probability at least $1-3\delta$, it holds that for every $t\in[T]$, 
$$|[{\mathbb{V}}W_k](s_t,a_t)-[\bar{\mathbb{V}}_tW_k](s_t,a_t)|\leq E_t,\quad \theta^*\in\widehat{\mathcal{C}}_t.$$
\end{lemma}

\section{Regret Analysis of \texttt{UCLK-C}}\label{sec:analysis}

Let us state the following regret bound of \texttt{UCLK-C} for linear mixture MDPs.
\begin{theorem}\label{thm:ub}
    Set $H\geq  2\cdot\mathrm{sp}(v^*)$, $\gamma=1-\sqrt{d}/\sqrt{HT}$, $N\geq \sqrt{{HT}/{d}}\log({\sqrt{T}}/{d\sqrt{H}})$, and $\lambda=1/B_\theta^2$.
    Then \texttt{UCLK-C} guarantees with probability at least $1-5\delta$ that for any linear mixture MDP with any initial state,
    $$\mathrm{Regret}(T)=\widetilde{\mathcal{O}}\left(d\sqrt{HT}+H\sqrt{dT}+ d^{7/4}HT^{1/4} \right)$$
    where the $\widetilde{\mathcal{O}}(\cdot)$ hides  logarithmic factors in $THB_\theta/\delta$.
\end{theorem}
By taking $H=2\cdot\mathrm{sp}(v^*)$, as a corollary of \Cref{thm:ub}, we deduce that
$$\mathrm{Regret}(T)=\widetilde{\mathcal{O}}\left(d\sqrt{\mathrm{sp}(v^*)T}\right)$$
where $\widetilde{\mathcal{O}}(\cdot)$ hides  logarithmic factors in $T\mathrm{sp}(v^*)B_\theta/\delta$. The rest of this section gives a proof overview of \Cref{thm:ub}.

Let us start by establishing convergence of the discounted extended value iteration procedure with clipping. For episode $k$, we consider the value functions $\widetilde V^{(n)}$ and $V^{(n)}$ for $n\in[N]$ (lines~\ref{alg:value-iteration}--\ref{alg:clipping} of \Cref{alg:UCLK-C}). Recall that $V^{(n)}$ is obtained from $\widetilde V^{(n)}$ after applying the clipping operation. It turns out that the clipping operation is a contraction map. 
\begin{lemma}\label{lem:contraction}
For any $n\in[N]$, it holds that
$$\max_{s\in\mathcal{S}}(V^{(n-1)}(s)-V^{(n)}(s))\leq \max_{s\in\mathcal{S}}(\widetilde V^{(n-1)}(s)-\widetilde V^{(n)}(s)).$$
\end{lemma}
Recall that $Q_k$ is the action-value function for the $k$th episode and that $V_k$ denotes the value function for the $k$th episode, obtained from clipping $\widetilde V_k$ with $\widetilde V_k(s) = \max_{a\in\mathcal{A}} Q_k(s,a)$ for $s\in\mathcal{S}$ (line~\ref{alg:value-functions} of \Cref{alg:UCLK-C}). Based on~\Cref{lem:contraction}, we prove the following lemma.
\begin{lemma}\label{lem:convergence-devi}
Suppose that $\theta^*\in \widehat{\mathcal{C}}_{t}$ for $t\in[T]$ where $\mathcal{C}_{t}$ is defined as in~\eqref{eq:confidence-set}. Then for each episode $k$ and $t_k\leq t<t_{k+1}-1$, it holds that
\begin{align*}
Q_k(s_t,a_t)
&\leq
r(s_t,a_t) + \gamma \max_{\theta\in\mathcal{C}_{k}}\langle \phi_{V_k}(s_t,a_t),\theta\rangle+\gamma^N.
\end{align*}
\end{lemma}
With \Cref{lem:convergence-devi}, we may provide a decomposition of the regret function as follows. Let $K_T$ denote the total number of distinct episodes over the horizon of $T$ time steps. For simplicity, we assume that the last time step of the last episode and that time step $T+1$ is the beginning of the $(K_T+1)$th episode, i.e., $t_{K_T+1} = T+1$. Then it follows from \Cref{lem:convergence-devi} that 
\begin{align*}
\begin{aligned}
\regret(T)
&= T\cdot J^*-\sum_{t=1}^T r(s_t,a_t) \\
&\leq T\gamma^N+\underbrace{\sum_{k=1}^{K_T}\sum_{t=t_k}^{t_{k+1}-1}\left(J^* -(1-\gamma)V_k(s_{t+1})\right)}_{I_1}+\underbrace{\sum_{k=1}^{K_T}\sum_{t=t_k}^{t_{k+1}-1}\left(V_k(s_{t+1})-Q_k(s_t,a_t)\right)}_{I_2}\\
& \quad+ \underbrace{\gamma\sum_{k=1}^{K_T}\sum_{t=t_k}^{t_{k+1}-1}\left(\langle \phi_{V_k}(s_t,a_t),\theta^*\rangle - V_k(s_{t+1})\right)}_{I_3}+\underbrace{\gamma\sum_{k=1}^{K_T}\sum_{t=t_k}^{t_{k+1}-1}\max_{\theta\in\mathcal{C}_k}\langle \phi_{V_k}(s_t,a_t),\theta-\theta^*\rangle}_{I_4}.
\end{aligned}
\end{align*}

\paragraph{Regret Term $I_1$} Recall that $V^*$ and $Q^*$ are the optimal value function and the optimal action-value function for the discounted-reward setting with discount factor $\gamma$. The following lemma proves that $V_k$ and $Q_k$ are optimistic estimators of $V^*$ and $Q^*$, respectively.
\begin{lemma}\label{lem:optimism}
Suppose that $\theta^*\in \widehat{\mathcal{C}}_{t}$ for $t\in[T]$ where $\widehat{\mathcal{C}}_{t}$ is defined as in~\eqref{eq:confidence-set}. Then for each episode $k$, $1/(1-\gamma)\geq V_k(s)\geq V^*(s)$ and $1/(1-\gamma)\geq Q_k(s,a)\geq Q^*(s,a)$.
\end{lemma}
\Cref{lem:optimism} implies that $J^*-(1-\gamma)V_k(s_{t+1})\leq J^* - (1-\gamma)V^*(s_{t+1})$. This can be further bounded above based on the following lemma.
\begin{lemma}{\em \citep[Lemma 2,][]{pmlr-v119-wei20c}}\label{lemma:span-bd}
Let $J^*$ and $v^*$ be the optimal average reward and the optimal bias function given in~\eqref{bellman-average}, and let $V^*$ be the optimal discounted value function given in~\eqref{bellman-discounted} with discount factor $\gamma\in[0,1)$. Then it holds that 
\begin{align*}
\max_{s\in\mathcal{S}}\left|J^*-(1-\gamma)V^*(s)\right|&\leq (1-\gamma)\mathrm{sp}(v^*),\quad
\mathrm{sp}(V^*) \leq 2\cdot\mathrm{sp}(v^*).
\end{align*}
\end{lemma}
\Cref{lemma:span-bd} offers a tool to bridge an infinite-horizon average-reward MDP and its discounted-reward MDP approximation. In particular, we deduce that $I_1\leq T(1-\gamma)\mathrm{sp}(v^*)\leq d\sqrt{\mathrm{sp}(v^*)T}.$
\paragraph{Regret Term $I_2$}
Note that $V_k(s_{t+1})\leq \widetilde V_k(s_{t+1})=Q_k(s_{t+1},a_{t+1})$ for $t\in[t_k:t_{k+1}-2]$, which leads to a telescoping structure. To be precise, we have
\begin{align*}
   I_2&\leq\sum_{k=1}^{K_T}\sum_{t=t_k}^{t_{k+1}-2}\left(Q_{k}(s_{t+1},a_{t+1}) - Q_k(s_t,a_t)\right)+ \sum_{k=1}^{K_T}\left(\frac{1}{1-\gamma} - Q_{k}(s_{t_{k+1}-1},a_{t_{k+1}-1}) \right)\\
    &=-\sum_{k=1}^{K_T}Q_{k}(s_{t_k},a_{t_k}) +\frac{K_T}{1-\gamma}.
\end{align*}
The following lemma gives an upper bound on the number of episodes $K_T$.
\begin{lemma}\label{lem:num-episodes}
If $\lambda=1/B_\theta^2$, then  $K(T)\leq 1+ d\log(1+TH^2B_\theta^2/d)$.
\end{lemma}
Then it follows form~\Cref{lem:num-episodes} that $I_2=\widetilde{\mathcal{O}}(d\sqrt{HT})$
where $\widetilde{\mathcal{O}}(\cdot)$ hides a logarithmic factor in $THB_\theta$.

\paragraph{Regret Term $I_3$}
We first observe that $$\langle \phi_{V_k}(s_t,a_t),\theta^*\rangle- \min_{s'\in\mathcal{S}}V_k(s')=\langle \phi_{W_k}(s_t,a_t),\theta^*\rangle$$
since $\langle \phi_{V_k}(s_t,a_t),\theta^*\rangle=[\mathbb{P}V_k](s_t,a_t)$. This implies that
$$I_3=\gamma\sum_{k=1}^{K_T}\sum_{t=t_k}^{t_{k+1}-1}\underbrace{\left(\langle \phi_{W_k}(s_t,a_t),\theta^*\rangle - W_k(s_{t+1})\right)}_{\eta_t}.$$
Then $\{\eta_t\}_{t=1}^T$ is a martingale difference sequence. Moreover, $|\eta_t|\leq H$ as $W_k(s)\in[0,H]$ for any $s\in\mathcal{S}$. Therefore, applying the Azuma-Hoeffding inequality, we deduce the following upper bound on $I_3$.
\begin{lemma}\label{lem:regret-term3}
It holds with probability at least $1-\delta$ that $I_3\leq H\sqrt{2T\log(1/\delta)}.$
\end{lemma}

\paragraph{Regret Term $I_4$} Let us sketch the idea of how the term $I_4$ can be bounded while a more rigorous proof is given in the appendix. Any $\theta\in\mathcal{C}_k$ induces a probability distribution, which implies that $\langle \phi_{V_k}(s_t,a_t),\theta-\theta^*\rangle=\langle \phi_{W_k}(s_t,a_t),\theta-\theta^*\rangle\in[-H,H]$. Moreover, assuming that $\theta^*\in \mathcal{C}_k$ based on \Cref{lem:confidence-ellipsoid},
\begin{align*}
    \langle \phi_{W_k}(s_t,a_t),\theta-\theta^*\rangle
    &\leq \left\|\phi_{W_k}(s_t,a_t)\right\|_{\widehat \Sigma_{t}^{-1}}\left(\|\theta-\widehat \theta_{t_k}\|_{\widehat\Sigma_{t}}+\|\widehat \theta_{t_k}-\theta^*\|_{\widehat\Sigma_{t}}\right)\\
    &\leq 2\left\|\phi_{W_k}(s_t,a_t)\right\|_{\widehat \Sigma_{t}^{-1}}\left(\|\theta-\widehat \theta_{t_k}\|_{\widehat\Sigma_{t_k}}+\|\widehat \theta_{t_k}-\theta^*\|_{\widehat\Sigma_{t_k}}\right)\\
    &\leq4\widehat \beta_T\bar\sigma_t\left\|\bar\sigma_t^{-1}\phi_{W_k}(s_t,a_t)\right\|_{\widehat \Sigma_{t}^{-1}}
\end{align*}
where the first inequality is from the Cauchy-Schwarz inequality, the second one holds because time step $t$ is in episode $k$ and thus $\det(\widehat \Sigma_t)\leq 2\det(\widehat \Sigma_{t_{k}})$, and the third one follows from $\theta^*,\theta\in\mathcal{C}_k$ and $\widehat\beta_{t_k}\leq \widehat \beta_T$. Then we may argue that $I_4$ is less than or equal to
\begin{align*}
 \sum_{k=1}^{K_T}\sum_{t=t_k}^{t_{k+1}-1}\min\left\{H, 4\widehat \beta_T\bar\sigma_t\left\|\bar\sigma_t^{-1}\phi_{W_k}(s_t,a_t)\right\|_{\widehat \Sigma_{t}^{-1}}\right\}
&\leq 4\widehat\beta_T\sum_{k=1}^{K_T}\sum_{t=t_k}^{t_{k+1}-1}\bar\sigma_t\min\left\{1, \left\|\bar\sigma_t^{-1}\phi_{W_k}(s_t,a_t)\right\|_{\widehat \Sigma_{t}^{-1}}\right\}\\
&\leq 4\widehat \beta_T\underbrace{\sqrt{\sum_{t=1}^T\bar\sigma_t^2}}_{J_1}\underbrace{\sqrt{\sum_{k=1}^{K_T}\sum_{t=t_k}^{t_{k+1}-1}\min\left\{1, \left\|\bar\sigma_t^{-1}\phi_{W_k}(s_t,a_t)\right\|_{\widehat \Sigma_{t}^{-1}}^2\right\}}}_{J_2}
\end{align*}
where the first inequality holds because $H\leq 4\widehat\beta_T\bar\sigma_t$ and the second one is by the Cauchy-Schwarz inequality. Here, we deduce from \citep[Lemma 11,][]{abbasi2011} that $J_2\leq \sqrt{2d\log(1+T/\lambda)}$. To get an upper bound on $J_1$, we show the following two lemmas.
\begin{lemma}\label{lem:I4-variance}
Suppose that the event of \Cref{lem:confidence-ellipsoid} holds and that
$$\sum_{k=1}^{K_T}\sum_{t=t_k}^{t_{k+1}-1}\left([\mathbb{P}W_k^2](s_t,a_t)-W_k^2(s_{t+1})\right)\leq H^2\sqrt{2T\log(1/\delta)}.$$
Then it holds that
$$\sum_{k=1}^{K_T}\sum_{t=t_k}^{t_{k+1}-1}[\mathbb{V}W_k](s_t,a_t)=\widetilde{\mathcal{O}}\left(HT+H^2d\sqrt{T}\right)$$
where $\widetilde{\mathcal{O}}(\cdot)$ hides logarithmic factors in $T/(\delta\lambda)$.
\end{lemma}

\Cref{lem:I4-variance} gives an upper bound on the variance term due to random transitions. We remark that the bound holds true even under our extended value iteration scheme with discounting and clipping.

\begin{lemma}\label{lem:I4-errors}
Suppose that the event of \Cref{lem:confidence-ellipsoid} holds. Then it holds that
$$\sum_{t=1}^T E_t=\widetilde{\mathcal{O}}\left(d^{3/2}H^2\sqrt{T}\right)$$
where $\widetilde{\mathcal{O}}(\cdot)$ hides logarithmic factors in $TH/\lambda$.
\end{lemma}
Combining the above results, it holds that
$$I_4=\widetilde{\mathcal{O}}\left(d\sqrt{HT}+H\sqrt{dT}+ d^{7/4}HT^{1/4} \right)$$
where $\widetilde{\mathcal{O}}(\cdot)$ hides logarithmic factors in $THB_\theta/\delta$.

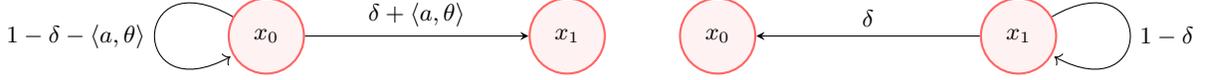
\begin{figure*}[!ht]
\begin{center}
\begin{tikzpicture}[
    roundnode/.style={circle, draw=red!60, fill=red!5, thick, minimum size = 10mm},
    squarednode/.style={rectangle, draw=red!60, fill=red!5, thick},
    ]
    \node[roundnode] (x_0) at (0,1) {\small$x_{0}$};
    \node[roundnode] (x_1) at (4,1) {\small$x_{1}$};

    \node[roundnode] (x'_0) at (6,1) {\small$x_{0}$};
    \node[roundnode] (x'_1) at (10,1) {\small$x_{1}$};
    \path [-stealth]
        (x_0) edge [] node[above] {\small $\delta+\langle a,\theta\rangle$} (x_1)
        (x'_1) edge [ ] node[above] {\small $\delta$} (x'_0)
        (x_0) edge [out=150,in=210, loop] node[left] {\small $1-\delta-\langle a,\theta\rangle$} ()
        (x'_1) edge [out=30,in=330, loop] node[right] {\small $1-\delta$} ();
\end{tikzpicture}%
\end{center}
\caption{Illustration of the Hard-to-Learn Infinite-Horizon MDP Instance}\label{fig:infinite}
\end{figure*}

\section{Regret Lower Bound}\label{sec:lb}

We deduce our regret lower bound based on the same hard-to-learn MDP instance due to~\cite{yuewu2022}, illustrated in \Cref{fig:infinite}. We derive the result based on the observation that for the MDP instance, the span of the underlying optimal bias function coincides with the diameter up to a constant multiplicative factor. 

There are two states $x_0$ and $x_1$ as in \Cref{fig:infinite}. The action space is given by $\mathcal{A}=\{-1,1\}^{d-1}$, and the reward function is given by $r(x_0,a)=0$ and $r(x_1,a)=1$ for any $a\in \mathcal{A}$.  We set the transition core $\bar\theta$ as
$$\bar \theta = \left(\frac{\theta}{\alpha}, \frac{1}{\beta}\right)\quad\text{where}\quad \theta\in\left\{-\frac{ \Delta}{d-1},\frac{\Delta}{d-1}\right\}^{d-1},$$
with $$\Delta=\frac{(d-1)}{45\sqrt{(2T\log 2)/(5\delta)}},$$
$\alpha=\sqrt{\Delta/((d-1)(1+\Delta))}$, and $\beta = \sqrt{1/(1+\Delta)}$. The feature vector is given by 
$\phi(x_0,a,x_0)=(-\alpha a, \beta (1-\delta))$, $\varphi(x_0,a,x_1)=(\alpha a, \beta\delta)$, $\varphi(x_1,a,x_0)=(0,\beta\delta)$, and $\varphi(x_1,a,x_1)=(0, \beta (1-\delta))$.

As a higher stationary probability at state $x_1$ means a larger average reward, choosing the action $a$ that satisfies $\langle a,\theta\rangle = \Delta$ is optimal. Hence, under the optimal policy, the stationary distribution is given by 
$$\left(\mu^*(x_0),\mu^*(x_1)\right)=\left(\frac{\delta}{2\delta+\Delta}, \frac{\delta+\Delta}{2\delta+\Delta}\right),$$
and therefore, the optimal average reward is given by
$$J^* =\frac{\delta+\Delta}{2\delta+\Delta}.$$

Moreover, we may observe that
\begin{align*}\left(v^*(x_0),v^*(x_1)\right)&=\left(0,\frac{1}{2\delta+\Delta}\right),\\
\left(q^*(x_0,a),q^*(x_1,a)\right)&=\left(\frac{\langle a,\theta\rangle - \Delta}{2\delta+\Delta},\frac{1}{2\delta+\Delta}\right)
\end{align*}
satisfy the Bellman optimality condition~\eqref{bellman-average}. In particular, we have
$$\mathrm{sp}(v^*) = \frac{1}{2\delta+ \Delta}.$$
When $T\geq 16(d-1)^2\delta^{-1}/2025$, we have $\Delta\leq \delta/3$.
Therefore, under this construction, we have
$$\frac{1}{3\delta}\leq \mathrm{sp}(v^*)\leq \frac{1}{2\delta}.$$

For this MDP instance, based on \citep[Theorem 5.5,][]{yuewu2022},we may deduce the following regret lower bound.
\begin{theorem}\label{thm:lb}
Suppose that $d\geq 2$ and $T\geq 16(d-1)^2\delta^{-1}/2025$. Then $\|\bar \theta\|_2 \leq 1+\delta/3$ for any $\theta\in\{-\Delta/(d-1),\Delta/(d-1)\}^{d-1}$ and $\|\varphi_F(x_0,a)\|_2,\|\varphi_F(x_1,a)\|_2\leq L$ for any $F:\mathcal{S}\to[0,L]$. Moreover, the span, $\mathrm{sp}(v^*)$, of the optimal bias function for the MDP instance satisfies $$\frac{1}{3\delta}\leq \mathrm{sp}(v^*)\leq \frac{1}{2\delta}.$$
Furthermore, for any algorithm, there exists $\theta\in\{-\Delta/(d-1),\Delta/(d-1)\}^{d-1}$ under which
$$\mathbb{E}\left[\regret(T)\right]\geq\frac{1}{2025}\sqrt{\frac{T}{\delta}}=\Omega\left(d\sqrt{\mathrm{sp}(v^*)T}\right).$$
\end{theorem}
As pointed out in~\citep[Remark C.1,][]{yuewu2022}, the regret lower bound in~\Cref{thm:lb} also translates to a regret lower bound for learning infinite-horizon average-reward linear MDPs of bounded span.

\section{Conclusion} 

This paper develops a provably efficient algorithm, \texttt{UCLK-C} for learning infinite-horizon average-reward linear mixture MDPs under the Bellman optimality condition. \texttt{UCLK-C} is the first algorithm that guarantees a nearly minimax optimal regret upper bound under the Bellman optimality condition. We establish this result based on our finding that discounted extended value iteration converges even with the additional clipping operation. We expect that this will be useful for infinite-horizon average-reward reinforcement under the Bellman optimality condition. We present some numerical results to test the computational performance of \texttt{UCLK-C} in the appendix.

Although we provide a nearly minimax optimal algorithm for linear mixture MDPs, there still exists a gap between the best-known regret upper and lower bounds for linear MDPs. To close the gap, one may attempt to extend the framework of this paper and other variance-aware parameter estimation schemes \citep[e.g.,][]{He2023} to the linear MDP setting. We propose this as an open problem.

\section{Acknowledgements}

We acknowledge the support of NSF via grant IIS-2007055 and the National Research Foundation of Korea (NRF) grant (No. RS-2024-00350703).
\bibliographystyle{apalike}
\bibliography{mybibfile}

\begin{thebibliography}{}

\bibitem[Abbasi-Yadkori et~al., 2019]{pmlr-v97-lazic19a}
Abbasi-Yadkori, Y., Bartlett, P., Bhatia, K., Lazic, N., Szepesvari, C., and
  Weisz, G. (2019).
\newblock {POLITEX}: Regret bounds for policy iteration using expert
  prediction.
\newblock In Chaudhuri, K. and Salakhutdinov, R., editors, {\em Proceedings of
  the 36th International Conference on Machine Learning}, volume~97 of {\em
  Proceedings of Machine Learning Research}, pages 3692--3702. PMLR.

\bibitem[Abbasi-yadkori et~al., 2011]{abbasi2011}
Abbasi-yadkori, Y., P\'{a}l, D., and Szepesv\'{a}ri, C. (2011).
\newblock Improved algorithms for linear stochastic bandits.
\newblock In Shawe-Taylor, J., Zemel, R., Bartlett, P., Pereira, F., and
  Weinberger, K., editors, {\em Advances in Neural Information Processing
  Systems}, volume~24. Curran Associates, Inc.

\bibitem[Agarwal et~al., 2021]{agarwal21}
Agarwal, A., Jiang, N., Kakade, S.~M., and Sun, W. (2021).
\newblock Reinforcement learning: Theory and algorithms.

\bibitem[Agarwal et~al., 2023]{Agarwal23}
Agarwal, A., Jin, Y., and Zhang, T. (2023).
\newblock Vo$q$l: Towards optimal regret in model-free rl with nonlinear
  function approximation.
\newblock In Neu, G. and Rosasco, L., editors, {\em Proceedings of Thirty Sixth
  Conference on Learning Theory}, volume 195 of {\em Proceedings of Machine
  Learning Research}, pages 987--1063. PMLR.

\bibitem[Agrawal and Jia, 2017]{NIPS2017_3621f145}
Agrawal, S. and Jia, R. (2017).
\newblock Optimistic posterior sampling for reinforcement learning: worst-case
  regret bounds.
\newblock In Guyon, I., Luxburg, U.~V., Bengio, S., Wallach, H., Fergus, R.,
  Vishwanathan, S., and Garnett, R., editors, {\em Advances in Neural
  Information Processing Systems}, volume~30. Curran Associates, Inc.

\bibitem[Auer et~al., 2008]{NIPS2008_e4a6222c}
Auer, P., Jaksch, T., and Ortner, R. (2008).
\newblock Near-optimal regret bounds for reinforcement learning.
\newblock In Koller, D., Schuurmans, D., Bengio, Y., and Bottou, L., editors,
  {\em Advances in Neural Information Processing Systems}, volume~21. Curran
  Associates, Inc.

\bibitem[Ayoub et~al., 2020]{pmlr-v119-ayoub20a}
Ayoub, A., Jia, Z., Szepesvari, C., Wang, M., and Yang, L. (2020).
\newblock Model-based reinforcement learning with value-targeted regression.
\newblock In III, H.~D. and Singh, A., editors, {\em Proceedings of the 37th
  International Conference on Machine Learning}, volume 119 of {\em Proceedings
  of Machine Learning Research}, pages 463--474. PMLR.

\bibitem[Bartlett and Tewari, 2009]{tewari12}
Bartlett, P.~L. and Tewari, A. (2009).
\newblock Regal: a regularization based algorithm for reinforcement learning in
  weakly communicating mdps.
\newblock In {\em Proceedings of the Twenty-Fifth Conference on Uncertainty in
  Artificial Intelligence}, UAI '09, page 35–42, Arlington, Virginia, USA.
  AUAI Press.

\bibitem[Boone and Zhang, 2024]{boone2024achievingtractableminimaxoptimal}
Boone, V. and Zhang, Z. (2024).
\newblock Achieving tractable minimax optimal regret in average reward mdps.

\bibitem[Bourel et~al., 2020]{pmlr-v119-bourel20a}
Bourel, H., Maillard, O., and Talebi, M.~S. (2020).
\newblock Tightening exploration in upper confidence reinforcement learning.
\newblock In III, H.~D. and Singh, A., editors, {\em Proceedings of the 37th
  International Conference on Machine Learning}, volume 119 of {\em Proceedings
  of Machine Learning Research}, pages 1056--1066. PMLR.

\bibitem[Cai et~al., 2020]{pmlr-v119-cai20d}
Cai, Q., Yang, Z., Jin, C., and Wang, Z. (2020).
\newblock Provably efficient exploration in policy optimization.
\newblock In III, H.~D. and Singh, A., editors, {\em Proceedings of the 37th
  International Conference on Machine Learning}, volume 119 of {\em Proceedings
  of Machine Learning Research}, pages 1283--1294. PMLR.

\bibitem[Dann et~al., 2018]{NEURIPS2018_5f0f5e5f}
Dann, C., Jiang, N., Krishnamurthy, A., Agarwal, A., Langford, J., and
  Schapire, R.~E. (2018).
\newblock On oracle-efficient pac rl with rich observations.
\newblock In Bengio, S., Wallach, H., Larochelle, H., Grauman, K.,
  Cesa-Bianchi, N., and Garnett, R., editors, {\em Advances in Neural
  Information Processing Systems}, volume~31. Curran Associates, Inc.

\bibitem[Du et~al., 2021]{bilinear}
Du, S., Kakade, S., Lee, J., Lovett, S., Mahajan, G., Sun, W., and Wang, R.
  (2021).
\newblock Bilinear classes: A structural framework for provable generalization
  in rl.
\newblock In Meila, M. and Zhang, T., editors, {\em Proceedings of the 38th
  International Conference on Machine Learning}, volume 139 of {\em Proceedings
  of Machine Learning Research}, pages 2826--2836. PMLR.

\bibitem[Filippi et~al., 2010]{Filippi10}
Filippi, S., Cappé, O., and Garivier, A. (2010).
\newblock Optimism in reinforcement learning and kullback-leibler divergence.
\newblock In {\em 2010 48th Annual Allerton Conference on Communication,
  Control, and Computing (Allerton)}, pages 115--122.

\bibitem[Fruit et~al., 2020]{fruit2020improvedanalysisucrl2empirical}
Fruit, R., Pirotta, M., and Lazaric, A. (2020).
\newblock Improved analysis of ucrl2 with empirical bernstein inequality.

\bibitem[Fruit et~al., 2018]{pmlr-v80-fruit18a}
Fruit, R., Pirotta, M., Lazaric, A., and Ortner, R. (2018).
\newblock Efficient bias-span-constrained exploration-exploitation in
  reinforcement learning.
\newblock In Dy, J. and Krause, A., editors, {\em Proceedings of the 35th
  International Conference on Machine Learning}, volume~80 of {\em Proceedings
  of Machine Learning Research}, pages 1578--1586. PMLR.

\bibitem[He et~al., 2023]{He2023}
He, J., Zhao, H., Zhou, D., and Gu, Q. (2023).
\newblock Nearly minimax optimal reinforcement learning for linear {M}arkov
  decision processes.
\newblock In Krause, A., Brunskill, E., Cho, K., Engelhardt, B., Sabato, S.,
  and Scarlett, J., editors, {\em Proceedings of the 40th International
  Conference on Machine Learning}, volume 202 of {\em Proceedings of Machine
  Learning Research}, pages 12790--12822. PMLR.

\bibitem[He et~al., 2024]{he2024sampleefficient}
He, J., Zhong, H., and Yang, Z. (2024).
\newblock Sample-efficient learning of infinite-horizon average-reward {MDP}s
  with general function approximation.
\newblock In {\em The Twelfth International Conference on Learning
  Representations}.

\bibitem[He et~al., 2021]{pmlr-v139-he21c}
He, J., Zhou, D., and Gu, Q. (2021).
\newblock Logarithmic regret for reinforcement learning with linear function
  approximation.
\newblock In Meila, M. and Zhang, T., editors, {\em Proceedings of the 38th
  International Conference on Machine Learning}, volume 139 of {\em Proceedings
  of Machine Learning Research}, pages 4171--4180. PMLR.

\bibitem[Hernandez-Lerma, 2012]{10.5555/516568}
Hernandez-Lerma, O. (2012).
\newblock {\em Adaptive Markov Control Processes}.
\newblock Springer New York, NY.

\bibitem[Hong et~al., 2024]{hong2024provablyefficientreinforcementlearning}
Hong, K., Zhang, Y., and Tewari, A. (2024).
\newblock Provably efficient reinforcement learning for infinite-horizon
  average-reward linear mdps.

\bibitem[Hu et~al., 2022]{Hu2022}
Hu, P., Chen, Y., and Huang, L. (2022).
\newblock Nearly minimax optimal reinforcement learning with linear function
  approximation.
\newblock In Chaudhuri, K., Jegelka, S., Song, L., Szepesvari, C., Niu, G., and
  Sabato, S., editors, {\em Proceedings of the 39th International Conference on
  Machine Learning}, volume 162 of {\em Proceedings of Machine Learning
  Research}, pages 8971--9019. PMLR.

\bibitem[Jia et~al., 2020]{pmlr-v120-jia20a}
Jia, Z., Yang, L., Szepesvari, C., and Wang, M. (2020).
\newblock Model-based reinforcement learning with value-targeted regression.
\newblock In Bayen, A.~M., Jadbabaie, A., Pappas, G., Parrilo, P.~A., Recht,
  B., Tomlin, C., and Zeilinger, M., editors, {\em Proceedings of the 2nd
  Conference on Learning for Dynamics and Control}, volume 120 of {\em
  Proceedings of Machine Learning Research}, pages 666--686. PMLR.

\bibitem[Jiang et~al., 2017]{jiang17}
Jiang, N., Krishnamurthy, A., Agarwal, A., Langford, J., and Schapire, R.~E.
  (2017).
\newblock Contextual decision processes with low {B}ellman rank are
  {PAC}-learnable.
\newblock In Precup, D. and Teh, Y.~W., editors, {\em Proceedings of the 34th
  International Conference on Machine Learning}, volume~70 of {\em Proceedings
  of Machine Learning Research}, pages 1704--1713. PMLR.

\bibitem[Jin et~al., 2020]{jin-linear-2020}
Jin, C., Yang, Z., Wang, Z., and Jordan, M.~I. (2020).
\newblock Provably efficient reinforcement learning with linear function
  approximation.
\newblock In Abernethy, J. and Agarwal, S., editors, {\em Proceedings of Thirty
  Third Conference on Learning Theory}, volume 125 of {\em Proceedings of
  Machine Learning Research}, pages 2137--2143. PMLR.

\bibitem[Kober et~al., 2013]{robotics}
Kober, J., Bagnell, J.~A., and Peters, J. (2013).
\newblock Reinforcement learning in robotics: A survey.
\newblock {\em The International Journal of Robotics Research},
  32(11):1238--1274.

\bibitem[Mnih et~al., 2015]{atari-mnih}
Mnih, V., Kavukcuoglu, K., Silver, D., Rusu, A.~A., Veness, J., Bellemare,
  M.~G., Graves, A., Riedmiller, M., Fidjeland, A.~K., Ostrovski, G., Petersen,
  S., Beattie, C., Sadik, A., Antonoglou, I., King, H., Kumaran, D., Wierstra,
  D., Legg, S., and Hassabis, D. (2015).
\newblock Human-level control through deep reinforcement learning.
\newblock {\em Nature}, 518(7540):529--533.

\bibitem[Modi et~al., 2020]{modi-linear-mixture}
Modi, A., Jiang, N., Tewari, A., and Singh, S. (2020).
\newblock Sample complexity of reinforcement learning using linearly combined
  model ensembles.
\newblock In Chiappa, S. and Calandra, R., editors, {\em Proceedings of the
  Twenty Third International Conference on Artificial Intelligence and
  Statistics}, volume 108 of {\em Proceedings of Machine Learning Research},
  pages 2010--2020. PMLR.

\bibitem[Ouyang et~al., 2017]{NIPS2017_51ef186e}
Ouyang, Y., Gagrani, M., Nayyar, A., and Jain, R. (2017).
\newblock Learning unknown markov decision processes: A thompson sampling
  approach.
\newblock In Guyon, I., Luxburg, U.~V., Bengio, S., Wallach, H., Fergus, R.,
  Vishwanathan, S., and Garnett, R., editors, {\em Advances in Neural
  Information Processing Systems}, volume~30. Curran Associates, Inc.

\bibitem[Puterman, 2014]{puterman2014markov}
Puterman, M.~L. (2014).
\newblock {\em Markov decision processes: discrete stochastic dynamic
  programming}.
\newblock John Wiley \& Sons.

\bibitem[Silver et~al., 2017]{go-silver}
Silver, D., Schrittwieser, J., Simonyan, K., Antonoglou, I., Huang, A., Guez,
  A., Hubert, T., Baker, L., Lai, M., Bolton, A., Chen, Y., Lillicrap, T., Hui,
  F., Sifre, L., van~den Driessche, G., Graepel, T., and Hassabis, D. (2017).
\newblock Mastering the game of go without human knowledge.
\newblock {\em Nature}, 550(7676):354--359.

\bibitem[Sun et~al., 2019]{pmlr-v99-sun19a}
Sun, W., Jiang, N., Krishnamurthy, A., Agarwal, A., and Langford, J. (2019).
\newblock Model-based rl in contextual decision processes: Pac bounds and
  exponential improvements over model-free approaches.
\newblock In Beygelzimer, A. and Hsu, D., editors, {\em Proceedings of the
  Thirty-Second Conference on Learning Theory}, volume~99 of {\em Proceedings
  of Machine Learning Research}, pages 2898--2933. PMLR.

\bibitem[Talebi and Maillard, 2018]{pmlr-v83-talebi18a}
Talebi, M.~S. and Maillard, O.-A. (2018).
\newblock Variance-aware regret bounds for undiscounted reinforcement learning
  in mdps.
\newblock In Janoos, F., Mohri, M., and Sridharan, K., editors, {\em
  Proceedings of Algorithmic Learning Theory}, volume~83 of {\em Proceedings of
  Machine Learning Research}, pages 770--805. PMLR.

\bibitem[Wang et~al., 2021]{wang2021optimism}
Wang, Y., Wang, R., Du, S.~S., and Krishnamurthy, A. (2021).
\newblock Optimism in reinforcement learning with generalized linear function
  approximation.
\newblock In {\em International Conference on Learning Representations}.

\bibitem[Wei et~al., 2021]{pmlr-v130-wei21d}
Wei, C.-Y., Jafarnia~Jahromi, M., Luo, H., and Jain, R. (2021).
\newblock Learning infinite-horizon average-reward mdps with linear function
  approximation.
\newblock In Banerjee, A. and Fukumizu, K., editors, {\em Proceedings of The
  24th International Conference on Artificial Intelligence and Statistics},
  volume 130 of {\em Proceedings of Machine Learning Research}, pages
  3007--3015. PMLR.

\bibitem[Wei et~al., 2020]{pmlr-v119-wei20c}
Wei, C.-Y., Jahromi, M.~J., Luo, H., Sharma, H., and Jain, R. (2020).
\newblock Model-free reinforcement learning in infinite-horizon average-reward
  {M}arkov decision processes.
\newblock In III, H.~D. and Singh, A., editors, {\em Proceedings of the 37th
  International Conference on Machine Learning}, volume 119 of {\em Proceedings
  of Machine Learning Research}, pages 10170--10180. PMLR.

\bibitem[Weisz et~al., 2021]{pmlr-v132-weisz21a}
Weisz, G., Amortila, P., and {Sz}epesv{\'a}ri, C. (2021).
\newblock Exponential lower bounds for planning in mdps with
  linearly-realizable optimal action-value functions.
\newblock In Feldman, V., Ligett, K., and Sabato, S., editors, {\em Proceedings
  of the 32nd International Conference on Algorithmic Learning Theory}, volume
  132 of {\em Proceedings of Machine Learning Research}, pages 1237--1264.
  PMLR.

\bibitem[Wu et~al., 2022]{yuewu2022}
Wu, Y., Zhou, D., and Gu, Q. (2022).
\newblock Nearly minimax optimal regret for learning infinite-horizon
  average-reward mdps with linear function approximation.
\newblock In Camps-Valls, G., Ruiz, F. J.~R., and Valera, I., editors, {\em
  Proceedings of The 25th International Conference on Artificial Intelligence
  and Statistics}, volume 151 of {\em Proceedings of Machine Learning
  Research}, pages 3883--3913. PMLR.

\bibitem[Yang and Wang, 2019]{yang-wang-2019}
Yang, L. and Wang, M. (2019).
\newblock Sample-optimal parametric q-learning using linearly additive
  features.
\newblock In Chaudhuri, K. and Salakhutdinov, R., editors, {\em Proceedings of
  the 36th International Conference on Machine Learning}, volume~97 of {\em
  Proceedings of Machine Learning Research}, pages 6995--7004. PMLR.

\bibitem[Yang and Wang, 2020]{yang-wang-2020}
Yang, L. and Wang, M. (2020).
\newblock Reinforcement learning in feature space: Matrix bandit, kernels, and
  regret bound.
\newblock In III, H.~D. and Singh, A., editors, {\em Proceedings of the 37th
  International Conference on Machine Learning}, volume 119 of {\em Proceedings
  of Machine Learning Research}, pages 10746--10756. PMLR.

\bibitem[Yurtsever et~al., 2020]{driving}
Yurtsever, E., Lambert, J., Carballo, A., and Takeda, K. (2020).
\newblock A survey of autonomous driving: Common practices and emerging
  technologies.
\newblock {\em IEEE Access}, 8:58443--58469.

\bibitem[Zanette et~al., 2020a]{pmlr-v108-zanette20a}
Zanette, A., Brandfonbrener, D., Brunskill, E., Pirotta, M., and Lazaric, A.
  (2020a).
\newblock Frequentist regret bounds for randomized least-squares value
  iteration.
\newblock In Chiappa, S. and Calandra, R., editors, {\em Proceedings of the
  Twenty Third International Conference on Artificial Intelligence and
  Statistics}, volume 108 of {\em Proceedings of Machine Learning Research},
  pages 1954--1964. PMLR.

\bibitem[Zanette et~al., 2020b]{pmlr-v119-zanette20a}
Zanette, A., Lazaric, A., Kochenderfer, M., and Brunskill, E. (2020b).
\newblock Learning near optimal policies with low inherent {B}ellman error.
\newblock In III, H.~D. and Singh, A., editors, {\em Proceedings of the 37th
  International Conference on Machine Learning}, volume 119 of {\em Proceedings
  of Machine Learning Research}, pages 10978--10989. PMLR.

\bibitem[Zhang and Ji, 2019]{NEURIPS2019_9e984c10}
Zhang, Z. and Ji, X. (2019).
\newblock Regret minimization for reinforcement learning by evaluating the
  optimal bias function.
\newblock In Wallach, H., Larochelle, H., Beygelzimer, A., d\textquotesingle
  Alch\'{e}-Buc, F., Fox, E., and Garnett, R., editors, {\em Advances in Neural
  Information Processing Systems}, volume~32. Curran Associates, Inc.

\bibitem[Zhang and Xie, 2023]{pmlr-v195-zhang23b}
Zhang, Z. and Xie, Q. (2023).
\newblock Sharper model-free reinforcement learning for average-reward markov
  decision processes.
\newblock In Neu, G. and Rosasco, L., editors, {\em Proceedings of Thirty Sixth
  Conference on Learning Theory}, volume 195 of {\em Proceedings of Machine
  Learning Research}, pages 5476--5477. PMLR.

\bibitem[Zhou and Gu, 2022]{NEURIPS2022_ebba182c}
Zhou, D. and Gu, Q. (2022).
\newblock Computationally efficient horizon-free reinforcement learning for
  linear mixture mdps.
\newblock In Koyejo, S., Mohamed, S., Agarwal, A., Belgrave, D., Cho, K., and
  Oh, A., editors, {\em Advances in Neural Information Processing Systems},
  volume~35, pages 36337--36349. Curran Associates, Inc.

\bibitem[Zhou et~al., 2021a]{zhou-mixture-finite-optimal}
Zhou, D., Gu, Q., and Szepesvari, C. (2021a).
\newblock Nearly minimax optimal reinforcement learning for linear mixture
  markov decision processes.
\newblock In Belkin, M. and Kpotufe, S., editors, {\em Proceedings of Thirty
  Fourth Conference on Learning Theory}, volume 134 of {\em Proceedings of
  Machine Learning Research}, pages 4532--4576. PMLR.

\bibitem[Zhou et~al., 2021b]{zhou2021nearlyminimaxoptimalreinforcement}
Zhou, D., Gu, Q., and Szepesvari, C. (2021b).
\newblock Nearly minimax optimal reinforcement learning for linear mixture
  markov decision processes, \emph{arXiv:2012.08507}.

\bibitem[Zhou et~al., 2021c]{pmlr-v139-zhou21a}
Zhou, D., He, J., and Gu, Q. (2021c).
\newblock Provably efficient reinforcement learning for discounted mdps with
  feature mapping.
\newblock In Meila, M. and Zhang, T., editors, {\em Proceedings of the 38th
  International Conference on Machine Learning}, volume 139 of {\em Proceedings
  of Machine Learning Research}, pages 12793--12802. PMLR.

\end{thebibliography}
\newpage
\appendix

\section{Confidence Ellipsoids}

In this section, we prove \Cref{lem:confidence-ellipsoid}. We start by analyzing the error term
$$\left|[\bar{\mathbb{V}}_t W_k](s_t,a_t) - [\mathbb{V} W_k](s_t,a_t)\right|.$$
Note that
\begin{align*}
    &\left|[\bar{\mathbb{V}}_t W_k](s_t,a_t) - [\mathbb{V} W_k](s_t,a_t)\right| \\
    & = \left|\left[\langle \phi_{W_k^2}(s_t,a_t),\widetilde\theta_t\rangle\right]_{[0,H^2]} - \langle \phi_{W_k^2}(s_t,a_t), \theta^* \rangle  + \langle \phi_{W_k}(s_t,a_t), \theta^* \rangle^2 - \left[\langle \phi_{W_k}(s_t,a_t),\widehat \theta_t\rangle\right]_{[0,H]}^2\right| \\
    & \leq \underbrace{\left|\left[\langle \phi_{W_k^2}(s_t,a_t),\widetilde\theta_t\rangle\right]_{[0,H^2]} - \langle \phi_{W_k^2}(s_t,a_t), \theta^* \rangle \right|}_{(a)}  + \underbrace{\left| \langle \phi_{W_k}(s_t,a_t), \theta^* \rangle^2 -  \left[\langle \phi_{W_k}(s_t,a_t),\widehat \theta_t\rangle\right]_{[0,H]}^2 \right|}_{(b)}.
\end{align*}
Let us consider term $(a)$ first. Since $\langle \phi_{W_k^2}(s_t,a_t), \theta^* \rangle\in[0,H^2]$, it follows that
$$(a)\leq \min\left\{H^2, \left|\langle \phi_{W_k^2}(s_t,a_t),\widetilde\theta_t\rangle-\langle \phi_{W_k^2}(s_t,a_t), \theta^* \rangle \right|\right\}.$$
Moreover, we may apply the Cauchy-Schwarz inequality to the right-hand side, and we obtain
$$(a)\leq \min\left\{H^2, \left\| \phi_{W_k^2}(s_t,a_t)\right\|_{\widetilde\Sigma_t^{-1}}\left\|\widetilde\theta_t-\theta^* \right\|_{\widetilde \Sigma_t}\right\}.$$
Next we consider term $(b)$. Again, since $\langle \phi_{W_k}(s_t,a_t), \theta^* \rangle\in[0,H]$, we have that $(b)\leq H^2$. Furthermore, 
\begin{align*}
(b)&= \left|\langle \phi_{W_k}(s_t,a_t), \theta^* \rangle +  \left[\langle \phi_{W_k}(s_t,a_t),\widehat \theta_t\rangle\right]_{[0,H]}\right|\cdot \left|\langle \phi_{W_k}(s_t,a_t), \theta^* \rangle -  \left[\langle \phi_{W_k}(s_t,a_t),\widehat \theta_t\rangle\right]_{[0,H]}\right|\\
&\leq 2H \left|\langle \phi_{W_k}(s_t,a_t), \theta^* \rangle -  \left[\langle \phi_{W_k}(s_t,a_t),\widehat \theta_t\rangle\right]_{[0,H]}\right|\\
&\leq 2H \left|\langle \phi_{W_k}(s_t,a_t), \theta^* \rangle -  \langle \phi_{W_k}(s_t,a_t),\widehat \theta_t\rangle\right|\\
&\leq 2H \left\|\phi_{W_k}(s_t,a_t)\right\|_{\widehat\Sigma_t^{-1}}\left\|\widehat\theta_t-\theta^*\right\|_{\widehat \Sigma_t}
\end{align*}
where the first and second inequalities hold because $\langle \phi_{W_k}(s_t,a_t), \theta^* \rangle\in[0,H]$ while the third inequality is due to the Cauchy-Schwarz inequality. Then we deduce that
$$(b)\leq \min\left\{H^2, 2H\left\|\phi_{W_k}(s_t,a_t)\right\|_{\widehat\Sigma_t^{-1}}\left\|\widehat\theta_t-\theta^*\right\|_{\widehat \Sigma_t}\right\}.$$
Therefore, we deduce that
\begin{equation}\label{eq:variance-error-term-bound-1}
\begin{aligned}
\left|[\bar{\mathbb{V}}_t W_k](s_t,a_t) - [\mathbb{V} W_k](s_t,a_t)\right|&\leq \min\left\{H^2, \left\| \phi_{W_k^2}(s_t,a_t)\right\|_{\widetilde\Sigma_t^{-1}}\left\|\widetilde\theta_t-\theta^* \right\|_{\widetilde \Sigma_t}\right\}\\
&\quad+ \min\left\{H^2, 2H\left\|\phi_{W_k}(s_t,a_t)\right\|_{\widehat\Sigma_t^{-1}}\left\|\widehat\theta_t-\theta^*\right\|_{\widehat \Sigma_t}\right\}.
\end{aligned}
\end{equation}

Let us consider the following two confidence ellipsoids for $\theta^*$:
\begin{align*}
     \check{\mathcal{C}}_t = \left\{\theta\in\mathcal{B}: \left\|\theta-\widehat{\theta}_t\right\|_{\widehat{\Sigma}_t} \leq \check{\beta}_t \right\},\quad \widetilde{\mathcal{C}}_t = \left\{ \theta\in\mathcal{B}: \left\| \theta-\widetilde{\theta}_t \right\|_{\widetilde{\Sigma}_t} \leq \widetilde{\beta}_t \right\}
\end{align*}
where \begin{align*}
\check{\beta}_t &= 8d\sqrt{\log(1+t/\lambda)\log(4t^2/\delta)} + 4\sqrt{d}\log(4t^2/\delta) + \sqrt{\lambda}B_\theta,\\
    \widetilde{\beta}_t &= 8H^2\sqrt{d \log(1 + tH^2/(d\lambda))\log(4t^2/\delta)}  + 4H^2\log(4t^2/\delta) + \sqrt{\lambda}B_\theta.
\end{align*}
First, we apply \Cref{lem:bernstein} to the linear bandit instance defined with \begin{equation}\label{bandit1}
(x_t,y_t)= (\bar\sigma_t^{-1}\phi_{W_k}(s_t,a_t), \bar\sigma_t^{-1}W_k(s_{t+1})),\quad \eta_t = y_t - \langle x_t, \theta^*\rangle
\end{equation}for $t\in[t_k:t_{k+1}-1]$ and $k\in[K_T]$. Then we deduce that
$$\widehat \Sigma_{t+1} = Z_t:=\lambda I_d + \sum_{i=1}^t x_i x_i^\top,\quad \widehat b_{t+1}=b_t:= \sum_{i=1}^t y_ix_i ,\quad \widehat \theta_{t+1} = \mu_t:= Z_t^{-1}b_t.$$
As $\mathcal{G}_t=\sigma(s_1,\ldots,s_t)$ is the $\sigma$-algebra generated by $s_1,\ldots, s_t$, we have
$\mathbb{E}\left[\eta_t\mid \mathcal{G}_t\right] = 0.$
Moreover, we deduce that
$$|\eta_t|\leq \frac{\sqrt{d}}{H}\left|W_k(s_{t+1})-\langle \phi_{W_k}(s_t,a_t),\theta^*\rangle\right|\leq \sqrt{d}$$
because $\bar\sigma_t^2 \geq H^2/d$ and it follows from $W_k(s)\in[0,H]$ for any $s\in\mathcal{S}$ that
$$-H\leq -\langle \phi_{W_k}(s_t,a_t),\theta^*\rangle\leq W_k(s_{t+1})-\langle \phi_{W_k}(s_t,a_t),\theta^*\rangle\leq W_k(s_{t+1})\leq H.$$
This also implies that $\mathbb{E}\left[\eta^2\mid \mathcal{G}_t\right]\leq d$ and that $$\|x_t\|_2\leq \frac{\sqrt{d}}{H}\|\phi_{W_k}(s_t,a_t)\|_2\leq \sqrt{d}.$$
Then it follows from \Cref{lem:bernstein} that with probability at least $1-\delta$, $\theta^*\in \check{\mathcal{C}}_t$ for all $t\in [T]$.

Next, we apply again \Cref{lem:bernstein} to the linear bandit instance defined with \begin{equation}\label{bandit2}
(x_t,y_t)=(\phi_{W_k^2}(s_t,a_t), W_k^2(s_{t+1})),\quad \eta_t = y_t - \langle x_t, \theta^*\rangle
\end{equation}
for $t\in[t_k:t_{k+1}-1]$ and $k\in[K_T]$. Then we deduce that
$$\widetilde \Sigma_{t+1} = Z_t:=\lambda I_d + \sum_{i=1}^t x_i x_i^\top,\quad \widetilde b_{t+1}=b_t:= \sum_{i=1}^t y_ix_i ,\quad \widetilde \theta_{t+1} = \mu_t:= Z_t^{-1}b_t.$$
As $\mathcal{G}_t=\sigma(s_1,\ldots,s_t)$ is the $\sigma$-algebra generated by $s_1,\ldots, s_t$, we have
$\mathbb{E}\left[\eta_t\mid \mathcal{G}_t\right] = 0.$ Since $W_k^2(s)\in[0,H^2]$ for any $s\in\mathcal{S}$, we have $|\eta_t|\leq H^2$ and thus $\mathbb{E}\left[\eta_t^2\mid \mathcal{G}_t\right]\leq H^2$. Moreover, $\|x_t\|_2=\|\phi_{W_k^2}(s_t,a_t)\|_2\leq H^2$. Then, by \Cref{lem:bernstein}, it hold with probability at least $1-\delta$ that $\theta^*\in \widehat{\mathcal{C}}_t$ for all $t\in [T]$.

Lastly, we apply \Cref{lem:bernstein} to the linear bandit instance defined with 
\begin{align}\label{bandit3}
\begin{aligned}
    x_t&= \bar\sigma_t^{-1} \phi_{W_k}(s_t,a_t),\\
    \eta_t &= \bar\sigma_t^{-1}\mathbf{1}\left\{\theta^*\in\check{\mathcal{C}}_t\cap \widetilde{\mathcal{C}}_t\right\}\left(W_k(s_{t+1})- \langle\phi_{W_k}(s_t,a_t),\theta^*\rangle\right),\\
    y_t &= \eta_t + \langle x_t,\theta^*\rangle
    \end{aligned}
\end{align}
where $\mathbf{1}\left\{\mathcal{E}\right\}$ denote the indicator function for a given event $\mathcal{E}$. Note that we still have $\mathbb{E}\left[\eta_t\mid \mathcal{G}_t\right]=0$ where $\mathcal{G}_t=\sigma(s_1,\ldots,s_t)$ is the $\sigma$-algebra generated by $s_1,\ldots, s_t$ because $\mathbf{1}\left\{\theta^*\in\check{\mathcal{C}}_t\cap \widetilde{\mathcal{C}}_t\right\}$ is $\mathcal{G}_t$-measurable. Moreover, as $\mathbf{1}\left\{\cdot\right\}\leq 1$, we have $|\eta_t|\leq\sqrt{d}$ and $\|x_t\|_2\leq \sqrt{d}$ as before. Let us consider $\mathbb{E}\left[\eta_t^2\mid\mathcal{G}_t\right]$. We obtain that 
\begin{align*}
\mathbb{E}\left[\eta_t^2\mid\mathcal{G}_t\right]&= \bar\sigma_t^{-2}\mathbf{1}\left\{\theta^*\in\check{\mathcal{C}}_t\cap \widetilde{\mathcal{C}}_t\right\}[\mathbb{V}W_k](s_t,a_t)\\
&\leq \bar\sigma_t^{-2}\mathbf{1}\left\{\theta^*\in\check{\mathcal{C}}_t\cap \widetilde{\mathcal{C}}_t\right\}\left([\bar{\mathbb{V}}_t W_k](s_t,a_t)+\min\left\{H^2, \left\| \phi_{W_k^2}(s_t,a_t)\right\|_{\widetilde\Sigma_t^{-1}}\left\|\widetilde\theta_t-\theta^* \right\|_{\widetilde \Sigma_t}\right\}\right)\\
&\quad+ \bar\sigma_t^{-2}\mathbf{1}\left\{\theta^*\in\check{\mathcal{C}}_t\cap \widetilde{\mathcal{C}}_t\right\}\min\left\{H^2, 2H\left\|\phi_{W_k}(s_t,a_t)\right\|_{\widehat\Sigma_t^{-1}}\left\|\widehat\theta_t-\theta^*\right\|_{\widehat \Sigma_t}\right\}\\
&\leq \bar\sigma_t^{-2}\left([\bar{\mathbb{V}}_t W_k](s_t,a_t)+\min\left\{H^2, \widetilde\beta_t\left\| \phi_{W_k^2}(s_t,a_t)\right\|_{\widetilde\Sigma_t^{-1}}\right\}+\min\left\{H^2, 2H\check\beta_t\left\|\phi_{W_k}(s_t,a_t)\right\|_{\widehat\Sigma_t^{-1}}\right\}\right)
\end{align*}
where the equality holds because $\mathbf{1}\left\{\theta^*\in\check{\mathcal{C}}_t\cap \widetilde{\mathcal{C}}_t\right\}$ is $\mathcal{G}_t$-measurable, the first inequality follows from~\eqref{eq:variance-error-term-bound-1}, and the second inequality holds due to our construction of $\check{\mathcal{C}}_t$ and $\widetilde{\mathcal{C}}_t$. Hence, we deduce that
$$\mathbb{E}\left[\eta_t^2\mid\mathcal{G}_t\right]\leq \bar\sigma_t^{-2} \left([\bar{\mathbb{V}}_t W_k](s_t,a_t)+E_t\right)\leq 1$$
where the first inequality is due to our choice of $E_t$ and the second inequality holds because $\bar\sigma_t^2\geq {[\bar{\mathbb{V}}_t W_k](s_t,a_t)+E_t}$. 

Taking the union bound, with probability at least $1-3\delta$, the statement of \Cref{lem:bernstein} holds for each of the three linear bandit instances given by~\eqref{bandit1},~\eqref{bandit2}, and~\eqref{bandit3}. We denote by $\mathcal{E}_0$ this event. Note that under event $\mathcal{E}_0$, we have $\theta^*\in\check{\mathcal{C}}_t\cap \widetilde{\mathcal{C}}_t$. In this case, we have that
$$|[{\mathbb{V}}W_k](s_t,a_t)-[\bar{\mathbb{V}}_tW_k](s_t,a_t)|\leq E_t.$$
Moreover, the outcome of \Cref{lem:bernstein} for the bandit instance~\eqref{bandit3} translates to the event that
$$
\theta^*\in \widehat{\mathcal{C}}_{t}=\left\{\theta\in\mathcal{B}: \|\theta-\widehat{\theta}_{t}\|_{\widehat{\Sigma}_{t}} \leq \widehat{\beta}_{t} \right\}$$
where 
\begin{align*}
    \widehat{\beta}_{t} = 8\sqrt{d\log(1+t/\lambda)\log(4t^2/\delta)}+4\sqrt{d}\log(4t^2/\delta)+\sqrt{\lambda}B_\theta,
\end{align*}
as required.

\section{Convergence of Value Iteration with Clipping}

In this section, we provide the proofs of \Cref{lem:contraction,lem:convergence-devi}.

\subsection{Proof of \Cref{lem:contraction}: Clipping as a Contraction Map}

Let $n\in[N]$, and let $s\in\mathcal{S}$. Then it is sufficient to argue that
$$V^{(n-1)}(s)- V^{(n)}(s)\leq \max_{s'\in\mathcal{S}}\left(\widetilde V^{(n-1)}(s')- \widetilde V^{(n)}(s')\right).$$
Recall that $V^{(i)}(s)=\min\{\widetilde V^{(i)}(s), \min_{s'\in\mathcal{S}}\widetilde V^{(i)}(s')+H\}$ for $i\in\{n-1,n\}$. Then there are two cases to consider depending on the value of $V^{(n)}(s)$.
\paragraph{Case I: $V^{(n)}(s)=\widetilde V^{(n)}(s)$.} Note that $V^{(n-1)}(s)\leq \widetilde V^{(n-1)}(s)$. This in turn implies that
$$V^{(n-1)}(s)- V^{(n)}(s)\leq\widetilde V^{(n-1)}(s)- \widetilde V^{(n)}(s)\leq \max_{s'\in\mathcal{S}}\left(\widetilde V^{(n-1)}(s')- \widetilde V^{(n)}(s')\right).$$

\paragraph{Case II: $V^{(n)}(s)=\min_{s'\in\mathcal{S}}\widetilde V^{(n)}(s')+H$.} Note that $V^{(n-1)}(s) \leq \min_{s'\in\mathcal{S}}\widetilde V^{(n-1)}(s')+H$. Then it follows that
\begin{align*}
V^{(n-1)}(s)- V^{(n)}(s)&\leq \min_{s'\in\mathcal{S}}\widetilde V^{(n-1)}(s')+H-\min_{s'\in\mathcal{S}}\widetilde V^{(n)}(s')-H= \min_{s'\in\mathcal{S}}\widetilde V^{(n-1)}(s')-\min_{s'\in\mathcal{S}}\widetilde V^{(n)}(s').
\end{align*}
Note that the right-most side satisfies
\begin{align*}
   &\min_{s'\in\mathcal{S}}\widetilde V^{(n-1)}(s')-\min_{s'\in\mathcal{S}}\widetilde V^{(n)}(s')\\   
   &=  -\max_{s'\in\mathcal{S}}\left(-\widetilde V^{(n-1)}(s')\right)+\max_{s'\in\mathcal{S}}\left(-\widetilde V^{(n)}(s')\right)\\
   &\leq  -\max_{s'\in\mathcal{S}}\left(-\widetilde V^{(n-1)}(s')\right)+\max_{s'\in\mathcal{S}}\left(-\widetilde V^{(n-1)}(s')\right)+\max_{s'\in\mathcal{S}}\left(\widetilde V^{(n-1)}(s')-\widetilde V^{(n)}(s')\right)\\
   &=\max_{s'\in\mathcal{S}}\left(\widetilde V^{(n-1)}(s')-\widetilde V^{(n)}(s')\right)
\end{align*}
where the inequality holds because $\max_p\{f(p)+g(p)\}\leq \max_p\{f(p)\} + \max_p\{g(p)\}$, as required.

\subsection{Proof of \Cref{lem:convergence-devi}: Convergence of Discounted Extended Value Iteration with Clipping}

We will first show the following lemma.
\begin{lemma}\label{lem:convergence-devi-1}
Let $N$ be the number of rounds for discounted extended value iteration with clipping. Then for any episode $k$, it holds that
$Q^{(N-1)}(s,a)-Q^{(N)}(s,a)\leq \gamma^{N-1}$ for any $(s,a)\in \mathcal{S}\times\mathcal{A}$.
\end{lemma}
\begin{proof}
We consider the discounted extended value iteration procedure with clipping for a fixed episode $k$. 
Note that for $n\geq 2$, we have
\begin{align*}
Q^{(n)}(s,a)  &=r(s,a)+\gamma \max_{\theta\in\mathcal{C}_k}\langle \phi_{V^{(n-1)}}(s,a),\theta\rangle,\\
Q^{(n-1)}(s,a)  &=r(s,a)+\gamma \max_{\theta\in\mathcal{C}_k}\langle \phi_{V^{(n-2)}}(s,a),\theta\rangle.
\end{align*}
This implies that for any $(s,a)\in\mathcal{S}\times\mathcal{A}$, 
\begin{align}\label{eq:lem:convergence-1}
\begin{aligned}
Q^{(n-1)}(s,a)-Q^{(n)}(s,a)
&=\gamma \left(\max_{\theta\in\mathcal{C}_k}\langle \phi_{V^{(n-2)}}(s,a),\theta\rangle-\max_{\theta\in\mathcal{C}_k}\langle \phi_{V^{(n-1)}}(s,a),\theta\rangle\right)\\
&\leq \gamma \max_{\theta\in\mathcal{C}_k}\left(\langle \phi_{V^{(n-2)}}(s,a),\theta\rangle-\langle \phi_{V^{(n-1)}}(s,a),\theta\rangle\right)\\
&= \gamma \max_{\theta\in\mathcal{C}_k}\langle \phi_{V^{(n-2)}-V^{(n-1)}}(s,a),\theta\rangle
\end{aligned}
\end{align}
where the inequality holds because $\max_p\{f(p)+g(p)\}\leq \max_p\{f(p)\} + \max_p\{g(p)\}$. Recall that for any $\theta\in\mathcal{C}_k$ induces a probability distribution with $\phi(s,a,s')$ given by
$\mathbb{P}_\theta(s'\mid s,a) = \langle \phi(s,a,s'),\theta\rangle.$ Then it follows that
\begin{align}\label{eq:lem:convergence-2}
\begin{aligned}
\gamma \langle \phi_{V^{(n-2)}-V^{(n-1)}}(s,a),\theta\rangle&= \gamma\mathbb{E}_{s'\sim \mathbb{P}_\theta(\cdot \mid s,a)}\left[V^{(n-2)}(s')-V^{(n-1)}(s')\right]\\
&\leq \gamma \max_{s'\in\mathcal{S}}\left(V^{(n-2)}(s')-V^{(n-1)}(s')\right)\\
&\leq \gamma \max_{s'\in\mathcal{S}}\left(\widetilde V^{(n-2)}(s')-\widetilde V^{(n-1)}(s')\right)
\end{aligned}
\end{align}
where the second inequality is due to~\Cref{lem:contraction}.
Combining~\eqref{eq:lem:convergence-1} and~\eqref{eq:lem:convergence-2}, we have
\begin{equation}\label{eq:lem:convergence-3}
\max_{(s,a)\in\mathcal{S}\times\mathcal{A}}\left(Q^{(n-1)}(s,a)-Q^{(n)}(s,a)\right)\leq\gamma  \max_{s\in \mathcal{S}}\left(\widetilde V^{(n-2)}(s)-\widetilde V^{(n-1)}(s)\right).
\end{equation}
Here, the right-hand side of~\eqref{eq:lem:convergence-3} can be further bounded as follows.
\begin{align}\label{eq:lem:convergence-4}
\begin{aligned}
\gamma \max_{s\in\mathcal{S}}\left(\widetilde V^{(n-2)}(s)-\widetilde V^{(n-1)}(s)\right)
&=  \gamma\max_{s\in\mathcal{S}}\left(\max_{a\in\mathcal{A}}Q^{(n-2)}(s,a)-\max_{a\in\mathcal{A}}Q^{(n-1)}(s,a)\right)\\
&\leq  \gamma\max_{(s,a)\in\mathcal{S}\times\mathcal{A}}\left(Q^{(n-2)}(s,a)-Q^{(n-1)}(s,a)\right)
\end{aligned}
\end{align}
where the inequality is due to $\max_{a'}\{f(a')+g(a')\}\leq \max_{a'}\{f(a')\} + \max_{a'}\{g(a')\}$ as before. 
Therefore, it follows that for any $n\geq 2$,
$$\max_{(s,a)\in\mathcal{S}\times\mathcal{A}}\left(Q^{(n-1)}(s,a)-Q^{(n)}(s,a)\right)\leq \gamma \max_{(s,a)\in\mathcal{S}\times\mathcal{A}}\left(Q^{(n-2)}(s,a)-Q^{(n-1)}(s,a)\right).$$
In particular, this implies that
\begin{align*}
\begin{aligned}
\max_{(s,a)\in\mathcal{S}\times\mathcal{A}}\left(Q^{(N-1)}(s,a)-Q^{(N)}(s,a)\right)&\leq \gamma^{N-1}\max_{(s,a)\in\mathcal{S}\times\mathcal{A}}\left(Q^{(0)}(s,a)-Q^{(1)}(s,a)\right)\\
&= \gamma^{N-1}\max_{(s,a)\in\mathcal{S}\times\mathcal{A}}\left(\frac{1}{1-\gamma}- r(s,a) -\frac{\gamma}{1-\gamma}\right)\\
&\leq \gamma^{N-1}
\end{aligned}
\end{align*}
where the last inequality holds because $0\leq r(s,a)\leq 1$.
\end{proof}

Based on~\Cref{lem:convergence-devi-1}, we complete the proof of Lemma~\ref{lem:convergence-devi}. Note that
\begin{align*}
Q^{(N)}(s_t,a_t)
&= r(s_t,a_t) + \gamma \max_{\theta\in\mathcal{C}_{k}}\langle\phi_{V^{(N-1)}}(s_t,a_t),\theta\rangle\\
&\leq r(s_t,a_t) + \gamma \max_{\theta\in\mathcal{C}_{k}}\langle\phi_{V^{(N)}}(s_t,a_t),\theta\rangle+\gamma \max_{\theta\in\mathcal{C}_{k}}\langle\phi_{V^{(N-1)}-V^{(N)}}(s_t,a_t),\theta\rangle\\
&\leq r(s_t,a_t) + \gamma \max_{\theta\in\mathcal{C}_{k}}\langle\phi_{V^{(N)}}(s_t,a_t),\theta\rangle  + \gamma \max_{s\in\mathcal{S}}\left(\widetilde V^{(N-1)}(s)-\widetilde V^{(N)}(s)\right) \\
&\leq r(s_t,a_t) + \gamma \max_{\theta\in\mathcal{C}_{k}}\langle\phi_{V^{(N)}}(s_t,a_t),\theta\rangle  + \gamma \max_{(s,a)\in\mathcal{S}\times\mathcal{A}}\left( Q^{(N-1)}(s,a)- Q^{(N)}(s,a)\right) \\
&\leq r(s_t,a_t) + \gamma \max_{\theta\in\mathcal{C}_{k}}\langle\phi_{V^{(N)}}(s_t,a_t),\theta\rangle  + \gamma^N
\end{align*}
where the first inequality applies the same argument as in~\eqref{eq:lem:convergence-2}, the second inequality follows the same argument as in~\eqref{eq:lem:convergence-4}, and the third inequality is due to \Cref{lem:convergence-devi-1}. Since $Q^{(N)}$ equals $Q_k$ and $V^{(N)}$ equals $V_k$, we have
$$Q_k(s_t,a_t)\leq r(s_t,a_t)+ \gamma \max_{\theta\in\mathcal{C}_{k}}\langle\phi_{V_k}(s_t,a_t),\theta\rangle + \gamma^N,$$
as required.

\section{Regret Analysis and Proofs}

In this section, we prove \Cref{lem:optimism} and \Cref{lem:num-episodes,lem:regret-term3,lem:I4-variance,lem:I4-errors}. Based on these results, in \Cref{sec:complete-ub}, we complete the proof of \Cref{thm:ub}.

\subsection{Proof of \Cref{lem:optimism}: Optimistic Estimators for Value Functions}

For a fixed episode $k$, we prove by induction on $n$ that for any $(s,a)\in\mathcal{S}\times\mathcal{A}$,
$$\frac{1}{1-\gamma}\geq V^{(n)}(s)\geq V^*(s),\quad \frac{1}{1-\gamma}\geq Q^{(n)}(s,a)\geq Q^*(s,a).$$
by induction on $n$. For $n=0$, it is trivial that 
\[V^{(0)}=\frac{1}{1-\gamma}\geq V^*(s), \quad {Q}^{(0)}(s,a)=\frac{1}{1-\gamma} \geq Q^*(s,a)\]
for every $(s,a)\in\mathcal{S}\times\mathcal{A}$. Next, we assume that for some $n\geq 0$, the inequalities $$\frac{1}{1-\gamma}\geq V^{(n)}(s)\geq V^*(s)\quad \text{and}\quad\frac{1}{1-\gamma}\geq {Q}^{(n)}(s,a)\geq Q^*(s,a)$$ hold for all $(s,a)\in\mathcal{S}\times\mathcal{A}$. First of all, since any $\theta\in\mathcal{C}_k$ induces a probability distribution, we get
$$Q^{(n+1)}(s,a)= r(s,a) + \gamma \max_{\theta\in\mathcal{C}_k}\langle \phi_{V^{(n)}}(s,a),\theta\rangle\leq 1+ \frac{\gamma}{1-\gamma} = \frac{1}{1-\gamma}$$
because $r(s,a)\leq 1$ and $V^{(n)}(s')\leq (1-\gamma)^{-1}$ for any $s'\in\mathcal{S}$. Then $$V^{(n+1)}\leq \widetilde V^{(n+1)}(s)=\max_{a\in\mathcal{A}}{Q}^{(n+1)}(s,a)\leq \frac{1}{1-\gamma}.$$

Next, we show that ${Q}^{(n+1)}(s,a)\geq Q^*(s,a)$ for any $(s,a)\in\mathcal{S}\times\mathcal{A}$.
Note that
\begin{align*}
   {Q}^{(n+1)}(s,a)&=r(s,a) + \gamma \max_{\theta\in\mathcal{C}_k}\langle \phi_{V^{(n)}}(s,a),\theta\rangle\\
    &\geq r(s,a) + \gamma \langle \phi_{V^{(n)}}(s,a),\theta^*\rangle\\
    &=r(s,a) + \gamma [\mathbb{P}V^{(n)}](s,a)\\
    &\geq r(s,a) + \gamma [\mathbb{P}V^*](s,a)\\
    &=Q^*(s,a)
\end{align*}
where the first inequality holds because $\theta^*\in \mathcal{C}_k=\widehat{\mathcal{C}}_{t_k}$, the second inequality is by the induction hypothesis, and the last equality is by the Bellman optimality condition~\eqref{bellman-discounted}.

Let us also consider ${V}^{(n+1)}$. Note that
\begin{align*}
    \widetilde {V}^{(n+1)}(s)-V^*(s) &= \max_{a\in\mathcal{A}}{Q}^{(n+1)}(s,a) - \max_{a\in\mathcal{A}} Q^*(s,a) \\
    &\geq \max_{a\in\mathcal{A}} Q^*(s,a)-\max_{a\in\mathcal{A}} Q^*(s,a)\\
    &= 0,
\end{align*}
where the inequality holds because $Q^{(n+1)}(s,a)\geq Q^*(s,a)$ for any $(s,a)\in\mathcal{S}\times\mathcal{A}$. This in turn implies that $\widetilde{V}^{(n+1)}(s)\geq V^*(s)$ for any $s\in\mathcal{S}$. This further implies that
\begin{align*}
    V^{(n+1)}(s)&=\min\left\{\widetilde{V}^{(n+1)}(s), \min_{s'\in\mathcal{S}}\widetilde{V}^{(n+1)}(s')+H \right\} \\
    & \geq \min\left\{V^*(s), \min_{s'\in\mathcal{S}}V^*(s')+H\right\} \\
    & = V^*(s),
\end{align*}
where the first inequality comes from our observation that $\widetilde{V}^{(n+1)}(s)\geq V^*(s)$ for any $s\in\mathcal{S}$ while the second equality holds because $\mathrm{sp}(V^*) \leq 2\cdot \mathrm{sp}(v^*) \leq H$, as supported by \cref{lemma:span-bd}.

Since $k$ was chosen arbitrarily, we conclude that in every episode, for all $n \in [N]$ and for all $(s, a) \in \mathcal{S} \times \mathcal{A}$, 
$$V^{(n)}(s) \geq V^*(s), \quad {Q}^{(n)}(s,a) \geq Q^*(s,a),$$
as required.

\subsection{Proof of \Cref{lem:num-episodes}: Upper Bound on the Number of Episodes}

Note that $\det({\widehat\Sigma}_{1})=\lambda^d$ because $\widehat\Sigma_1=\lambda I_d$. To upper bound
$\det(\widehat\Sigma_{T+1})$, we apply the following lemma. 
\begin{lemma} \label{lem:determinant-trace}
{\em \citep[Lemma 10,][]{abbasi2011}} For any $x_1,\ldots, x_T\in\mathbb{R}^d$ such that $\|{x}_t\|_2\leq L$, let ${A}_1=\lambda I_d$ and ${A}_{t+1} = \lambda I_d + \sum_{i=1}^{t}{x}_i{x}_i^\top$ for $t\geq 1$. Then
$$\det(\widehat\Sigma_{T+1})\leq \left(\lambda + \frac{TL^2}{d}\right)^d.$$
\end{lemma}
Recall that
$$\widehat\Sigma_{T+1}=\lambda I_d + \sum_{k=1}^{K_T}\sum_{t=t_{k}}^{t_{k+1}-1}\phi_{W_{k}}(s_{t},a_{t})\phi_{W_{k}}(s_{t},a_{t})^\top.$$
Here, we have $\|\phi_{W_{k}}(s_{t},a_{t})\|_2\leq H$ as $W_{k}(s)\in[0,H]$ for any $s\in\mathcal{S}$. By applying \Cref{lem:determinant-trace}, we get
$$\det(\widehat\Sigma_{T+1})\leq \left(\lambda+ \frac{TH^2}{d}\right)^d.$$
As $\lambda = 1/B_\theta^2$, it follows that
$$\frac{\det(\widehat\Sigma_{T+1})}{\det(\widehat\Sigma_1)}\leq \left(1+ \frac{TH^2}{d\lambda }\right)^d= \left(1+ \frac{TH^2B_\theta^2}{d}\right)^d.$$
Moreover, note that
$$\det(\widehat\Sigma_{T+1}) \geq \det(\widehat\Sigma_T)\geq \det(\widehat\Sigma_{t_{K_T}})\geq \cdots \geq 2^{K_T-1}\det(\widehat\Sigma_{t_{1}}),$$
implying in turn that
$$K_T\leq 1+ \log_2(\det(\widehat\Sigma_{T+1})/\det(\widehat\Sigma_{t_{1}}))\leq 1+ d\log_2(1+TH^2B_\theta^2/d),$$
as required.

\subsection{Proof of \Cref{lem:regret-term3}: Martingale Difference Sequence}

The term $I_3$ is a sum of martingale difference sequence $\{\eta_t\}_{t=1}^{\infty}$ with regard to a filtration $\{\mathcal{G}_t\}_{t=0}^{\infty}$, where $$\eta_t=\langle \phi_{V_k}(s_t,a_t),\theta^*\rangle - V_k(s_{t+1})$$
and $\mathcal{G}_t=\sigma(s_1,\ldots,s_t)$ is the $\sigma$-algebra generated by $s_1,\ldots, s_t$.
for $t\in[t_k:t_{k+1}-1]$. This is because  $\eta_t$ is $\mathcal{G}_{t+1}$-measurable, $\mathbb{E}[|\eta_t|] < \infty$, and $\mathbb{E}[\eta_t|\mathcal{G}_{t}]=0$, which we will show in the following paragraphs. In fact, we have
$$\eta_t=\langle \phi_{V_k}(s_t,a_t),\theta^*\rangle - V_k(s_{t+1})=\langle \phi_{W_k}(s_t,a_t),\theta^*\rangle - W_k(s_{t+1}),$$
which implies that
$$|\eta_t|\leq \mathrm{sp}(W_k)\leq H.$$
\begin{lemma}{\em (Azuma-Hoeffding inequality)}\label{lemma:azuma-hoeffding}
     Let $\{X_k\}_{k=0}^{\infty}$ be a discrete-parameter real-valued martingale sequence such that for every $k\in\mathbb{N}$, the condition $|X_k-X_{k-1}|\leq \mu$ holds for some non-negative constant $\mu$. Then with probability at least $1-\delta$, we have
\[X_n-X_0 \leq \mu\sqrt{2n\log(1/\delta)}.\] 
\end{lemma}
Since $X_t=\sum_{n=1}^t \eta_t$ for $t\geq 1$ and $X_0$ give rise to a martingale sequence with $|\eta_t|\leq H$, it follows from~\Cref{lemma:azuma-hoeffding} that
\begin{align*}
    I_3 \leq H\sqrt{2T\log(1/\delta)}
\end{align*}
holds with probability at least $1-\delta$.

\subsection{Proof of \Cref{lem:I4-variance}: Variance Term}

Note that
\begin{align*}
&\sum_{k=1}^{K_T}\sum_{t=t_k}^{t_{k+1}-1}[\mathbb{V}W_k](s_t,a_t)\\
&=\sum_{k=1}^{K_T}\sum_{t=t_k}^{t_{k+1}-1}\left([\mathbb{P}W_k^2](s_t,a_t)-([\mathbb{P}W_k](s_t,a_t))^2\right)\\
&=\sum_{k=1}^{K_T}\sum_{t=t_k}^{t_{k+1}-1}\left([\mathbb{P}W_k^2](s_t,a_t)-W_k^2(s_{t+1})\right) + \sum_{k=1}^{K_T}\sum_{t=t_k}^{t_{k+1}-1}\left(W_k^2(s_{t+1})-W_k^2(s_{t})\right)\\
&\quad +\sum_{k=1}^{K_T}\sum_{t=t_k}^{t_{k+1}-1}\left(W_k^2(s_t)-([\mathbb{P}W_k](s_t,a_t))^2\right)\\
&=\sum_{k=1}^{K_T}\sum_{t=t_k}^{t_{k+1}-1}\left([\mathbb{P}W_k^2](s_t,a_t)-W_k^2(s_{t+1})\right) + \sum_{k=1}^{K_T}\sum_{t=t_k}^{t_{k+1}-1}\left(W_k^2(s_t)-([\mathbb{P}W_k](s_t,a_t))^2\right)\\
&\quad + \sum_{k=1}^{K_T}\left(W_k^2(s_{t_{k+1}}) - W_k^2(s_{t_k})\right)\\
&\leq \underbrace{\sum_{k=1}^{K_T}\sum_{t=t_k}^{t_{k+1}-1}\left([\mathbb{P}W_k^2](s_t,a_t)-W_k^2(s_{t+1})\right)}_{(i)} + \underbrace{\sum_{k=1}^{K_T}\sum_{t=t_k}^{t_{k+1}-1}\left(W_k^2(s_t)-([\mathbb{P}W_k](s_t,a_t))^2\right)}_{(ii)} + H^2 K_T.
\end{align*}
By the condition of the lemma, we have
$$(i)\leq H^2\sqrt{2T\log(1/\delta)}.$$
Let us consider the term $(ii)$. Note that
\begin{align*}
(ii)&=\sum_{k=1}^{K_T}\sum_{t=t_k}^{t_{k+1}-1}\left(W_k(s_t)-[\mathbb{P}W_k](s_t,a_t)\right)\left(W_k(s_t)+[\mathbb{P}W_k](s_t,a_t)\right).
\end{align*}
Here, we have
\begin{align*}
W_k(s_t) &= V_k(s_t) - \min_{s'\in\mathcal{S}}V_k(s')\\
&\leq \widetilde V_k(s_t) - \min_{s'\in\mathcal{S}}V_k(s')\\
&=Q_k(s_t,a_t)- \min_{s'\in\mathcal{S}}V_k(s')\\
&\leq 1+ \gamma^N + \max_{\theta\in\mathcal{C}_{k}}\langle \phi_{V_k}(s_t,a_t),\theta\rangle - \min_{s'\in\mathcal{S}}V_k(s')\\
&\leq 2 + \max_{\theta\in\mathcal{C}_{k}}\langle \phi_{W_k}(s_t,a_t),\theta\rangle
\end{align*}
where the second inequality holds because $\gamma\leq 1$ and $r(s_t,a_t)\leq 1$ and the third inequality holds because $\gamma\leq 1$ and any $\theta\in\mathcal{C}_k$ induces a probability distribution. Then it follows that
\begin{align*}
(ii)&\leq 2\sum_{k=1}^{K_T}\sum_{t=t_k}^{t_{k+1}-1}\left(W_k(s_t)+[\mathbb{P}W_k](s_t,a_t)\right) +\sum_{k=1}^{K_T}\sum_{t=t_k}^{t_{k+1}-1}\left(\max_{\theta\in\mathcal{C}_{k}}\langle \phi_{W_k}(s_t,a_t),\theta-\theta^*\rangle\right)\left(W_k(s_t)+[\mathbb{P}W_k](s_t,a_t)\right)\\
&\leq 4HT + 2H\sum_{k=1}^{K_T}\sum_{t=t_k}^{t_{k+1}-1}\left|\max_{\theta\in\mathcal{C}_{k}}\langle \phi_{W_k}(s_t,a_t),\theta-\theta^*\rangle\right|.
\end{align*}
Note that
\begin{align*}
\left|\max_{\theta\in\mathcal{C}_{k}}\langle \phi_{W_k}(s_t,a_t),\theta-\theta^*\rangle\right|&\leq \max_{\theta\in\mathcal{C}_{k}}\left|\langle \phi_{W_k}(s_t,a_t),\theta-\theta^*\rangle\right|\leq \max_{\theta\in\mathcal{C}_{k}}\left\|\phi_{W_k}(s_t,a_t)\right\|_{\widehat \Sigma_{t}^{-1}}\|\theta-\theta^*\|_{\widehat\Sigma_{t}}.
\end{align*}
To bound the right-most side, we need the following lemma.
\begin{lemma}\label{lem:determinant1}{\em \citep[Lemma 12]{abbasi2011}}
Let $A,B\in\mathbb{R}^{d\times d}$ be positive semidefinite matrices such that $A\succeq B$. Then for any $x\in\mathbb{R}^d$, we have 
$\|x\|_A\leq \|x\|_B\sqrt{\det(A)/\det(B)}$.
\end{lemma}
By \Cref{lem:determinant1}, we have $\|\theta-\theta^*\|_{\widehat\Sigma_{t}}\leq 2 \|\theta-\theta^*\|_{\widehat\Sigma_{t_k}}$. Therefore, 
\begin{align*}
\left|\max_{\theta\in\mathcal{C}_{k}}\langle \phi_{W_k}(s_t,a_t),\theta-\theta^*\rangle\right|&\leq \sqrt{2}\max_{\theta\in\mathcal{C}_{k}}\left\|\phi_{W_k}(s_t,a_t)\right\|_{\widehat \Sigma_{t}^{-1}}\left(\|\theta-\widehat\theta_{t_k}\|_{\widehat\Sigma_{t_k}}+\|\widehat\theta_{t_k}-\theta^*\|_{\widehat\Sigma_{t_k}}\right)\leq 2\sqrt{2}\widehat \beta_T\left\|\phi_{W_k}(s_t,a_t)\right\|_{\widehat \Sigma_{t}^{-1}}
\end{align*}
where the second inequality holds because $\theta,\theta^*\in\mathcal{C}_k$.
Meanwhile, we already know that $\left|\langle \phi_{W_k}(s_t,a_t),\theta-\theta^*\rangle\right|\leq H$ for any $\theta$ that induces a probability distribution. Then it follows that
\begin{align*}
\sum_{k=1}^{K_T}\sum_{t=t_k}^{t_{k+1}-1}\left|\max_{\theta\in\mathcal{C}_{k}}\langle \phi_{W_k}(s_t,a_t),\theta-\theta^*\rangle\right|&\leq \sum_{k=1}^{K_T}\sum_{t=t_k}^{t_{k+1}-1}\min\left\{H,2\sqrt{2}\widehat \beta_T\left\|\phi_{W_k}(s_t,a_t)\right\|_{\widehat \Sigma_{t}^{-1}}\right\}\\
&\leq \sum_{k=1}^{K_T}\sum_{t=t_k}^{t_{k+1}-1}2\sqrt{2}\widehat \beta_T\bar\sigma_t\min\left\{1,\left\|\bar\sigma_t^{-1}\phi_{W_k}(s_t,a_t)\right\|_{\widehat \Sigma_{t}^{-1}}\right\}\\
&\leq 4\sqrt{2}H\widehat \beta_T\sum_{k=1}^{K_T}\sum_{t=t_k}^{t_{k+1}-1}\min\left\{1,\left\|\bar\sigma_t^{-1}\phi_{W_k}(s_t,a_t)\right\|_{\widehat \Sigma_{t}^{-1}}\right\}\\
&\leq 4\sqrt{2}H\widehat \beta_T\sqrt{T}\sqrt{\sum_{k=1}^{K_T}\sum_{t=t_k}^{t_{k+1}-1}\min\left\{1,\left\|\bar\sigma_t^{-1}\phi_{W_k}(s_t,a_t)\right\|_{\widehat \Sigma_{t}^{-1}}^2\right\}}
\end{align*}
where the second inequality is from $H\leq 4\widehat \beta_T\bar\sigma_t$, the third inequality holds because we have $[\bar{\mathbb{V}}_tW_k](s_t,a_t)\leq H^2$ and $E_t\leq 2H^2$, which implies that $\bar\sigma_t\leq \sqrt{\max\{H^2/d, 3H^2\}}\leq 2H$, and the last inequality is implied by the Cauchy-Schwarz inequality. To bound the right-most side, we need the following lemma.
\begin{lemma}\label{lem:abbasi lemma 11}{\em \citep[Lemma 11,][]{abbasi2011}.}
Suppose $x_1,\ldots, x_t\in \mathbb{R}^d$ and $\|x_s\|_2\leq L$ for any $1\leq s\leq t$. 
Let ${V}_t = \lambda {I}_d + \sum_{i=1}^{t} x_i x_i^\top$ for some $\lambda > 0$.
Then 
    \begin{align*}
    \sum_{i=1}^t\min\left\{1,\|x_i\|_{V_{i-1}^{-1}}^2\right\} \leq 2d\log\left( 1 + \frac{t  L^2}{d\lambda} \right).
    \end{align*}
\end{lemma}
Since $\left\|\bar\sigma_t^{-1}\phi_{W_k}(s_t,a_t)\right\|_{2}\leq \sqrt{d}$, \Cref{lem:abbasi lemma 11} implies that
\begin{equation}\label{eq:vectorsum}
\sum_{k=1}^{K_T}\sum_{t=t_k}^{t_{k+1}-1}\min\left\{1,\left\|\bar\sigma_t^{-1}\phi_{W_k}(s_t,a_t)\right\|_{\widehat \Sigma_{t}^{-1}}^2\right\}\leq 2d\log (1+T/\lambda),
\end{equation}
implying in turn that
$$(ii)\leq 4HT + 16H^2 \widehat \beta_T\sqrt{dT\log(1+T/\lambda)}.$$
Consequently, 
$$\sum_{k=1}^{K_T}\sum_{t=t_k}^{t_{k+1}-1}[\mathbb{V}W_k](s_t,a_t)\leq (i)+(ii)+ H^2K_T\leq 4HT +16H^2 \widehat \beta_T\sqrt{dT\log(1+T/\lambda)}+H^2K_T.$$
Then it follows from our choice of $\widehat \beta_T$ and \Cref{lem:num-episodes} that
$$\sum_{k=1}^{K_T}\sum_{t=t_k}^{t_{k+1}-1}[\mathbb{V}W_k](s_t,a_t)=\widetilde{\mathcal{O}}\left(HT+H^2d\sqrt{T}\right)$$
where $\widetilde{\mathcal{O}}(\cdot)$ hides logarithmic factors in $T/(\delta\lambda)$.

\subsection{Proof of \Cref{lem:I4-errors}: Cumulative Error in Estimating the Variance}

Note that
\begin{equation*}
\begin{aligned}
\sum_{t=1}^TE_t=\underbrace{\sum_{k=1}^{K_T}\sum_{t=t_k}^{t_{k+1}-1}  \min\left\{H^2, 2H\check{\beta}_{t}\|\phi_{W_k}(s_t,a_t)\|_{\widehat{\Sigma}_{t}^{-1}}\right\}}_{(I)} + \underbrace{\sum_{k=1}^{K_T}\sum_{t=t_k}^{t_{k+1}-1}\min\left\{H^2, \widetilde{\beta}_{t} \|\phi_{W_k^2}(s_t,a_t)\|_{\widetilde{\Sigma}_t^{-1}} \right\}}_{(II)}.
\end{aligned}
\end{equation*}
Term $(I)$ can be bounded as follows:
\begin{align*}
(I)&\leq\sum_{k=1}^{K_T}\sum_{t=t_k}^{t_{k+1}-1}  2H\check\beta_t\bar\sigma_t\min\left\{1, \|\bar\sigma_t^{-1}\phi_{W_k}(s_t,a_t)\|_{\widehat{\Sigma}_{t}^{-1}}\right\}\\
&\leq4H^2\check\beta_T \sum_{k=1}^{K_T}\sum_{t=t_k}^{t_{k+1}-1}  \min\left\{1, \|\bar\sigma_t^{-1}\phi_{W_k}(s_t,a_t)\|_{\widehat{\Sigma}_{t}^{-1}}\right\}\\
&\leq 4H^2\check\beta_T\sqrt{T}\sqrt{\sum_{k=1}^{K_T}\sum_{t=t_k}^{t_{k+1}-1}\min\left\{1,\left\|\bar\sigma_t^{-1}\phi_{W_k}(s_t,a_t)\right\|_{\widehat \Sigma_{t}^{-1}}^2\right\}}\\
&\leq 8H^2\check\beta_T\sqrt{dT\log(1+T/\lambda)}
\end{align*}
where the first inequality is due to $H\leq 2\check\beta_t \bar\sigma_t$, the second inequality holds because  we have $[\bar{\mathbb{V}}_tW_k](s_t,a_t)\leq H^2$ and $E_t\leq 2H^2$, which implies that $\bar\sigma_t\leq \sqrt{\max\{H^2/d, 3H^2\}}\leq 2H$, the third inequality is due to the Cauchy-Schwarz inequality, and the last one follows from~\eqref{eq:vectorsum}.

Term $(II)$ can be bounded as follows:
\begin{align*}
(II)&\leq\widetilde\beta_T\sum_{k=1}^{K_T}\sum_{t=t_k}^{t_{k+1}-1}  \min\left\{1, \|\phi_{W_k^2}(s_t,a_t)\|_{\widetilde{\Sigma}_{t}^{-1}}\right\}\\
&\leq \widetilde\beta_T\sqrt{T}\sqrt{\sum_{k=1}^{K_T}\sum_{t=t_k}^{t_{k+1}-1}  \min\left\{1, \|\phi_{W_k^2}(s_t,a_t)\|_{\widetilde{\Sigma}_{t}^{-1}}^2\right\}}\\
&\leq 2\widetilde\beta_T\sqrt{dT\log(1+TH^2/\lambda)}
\end{align*}
where the first inequality holds because $H^2\leq \widetilde\beta_t\leq \widetilde\beta_T$, the second inequality is by the Cauchy-Schwarz inequality, and the third one follows from \Cref{lem:abbasi lemma 11}.

Therefore, it holds that
$$\sum_{t=1}^T E_t \leq 8H^2\check\beta_T\sqrt{dT\log(1+T/\lambda)}+2\widetilde\beta_T\sqrt{dT\log(1+TH^2/\lambda)}.$$
Due to our choice of $\check\beta_T$ and $\widetilde\beta_T$, we have
$$\sum_{t=1}^T E_t=\widetilde{\mathcal{O}}\left(d^{3/2}H^2\sqrt{T}\right)$$
where $\widetilde{\mathcal{O}}(\cdot)$ hides logarithmic factors in $TH/\lambda$, as required.

\subsection{Completing the Proof of \Cref{thm:ub}: Regret Upper Bound of \texttt{UCLK-C}}\label{sec:complete-ub}

We first provide an upper bound on $I_4$.
Let $\theta\in\mathcal{C}_k$. Since $\theta$ induces a probability distribution $\mathbb{P}_\theta$, we have
\begin{align*}
\langle \phi_{V_k}(s_t,a_t),\theta-\theta^*\rangle&=\mathbb{E}_{s'\sim\mathbb{P}_\theta(\cdot\mid s_t,a_t)}[V_k(s')]- \mathbb{E}_{s'\sim\mathbb{P}(\cdot\mid s_t,a_t)}[V_k(s')]\\
&=\mathbb{E}_{s'\sim\mathbb{P}_\theta(\cdot\mid s_t,a_t)}[V_k(s')]-\min_{s'\in\mathcal{S}}V_k(s')- \mathbb{E}_{s'\sim\mathbb{P}(\cdot\mid s_t,a_t)}[V_k(s')]+\min_{s'\in\mathcal{S}}V_k(s')\\
&=\mathbb{E}_{s'\sim\mathbb{P}_\theta(\cdot\mid s_t,a_t)}[W_k(s')]- \mathbb{E}_{s'\sim\mathbb{P}(\cdot\mid s_t,a_t)}[W_k(s')]\\
&=\langle \phi_{W_k}(s_t,a_t),\theta-\theta^*\rangle.
\end{align*}
Moreover, assuming that $\theta^*\in \mathcal{C}_k$ based on \Cref{lem:confidence-ellipsoid}, we have
\begin{align*}
    \langle \phi_{W_k}(s_t,a_t),\theta-\theta^*\rangle
    &\leq \left\|\phi_{W_k}(s_t,a_t)\right\|_{\widehat \Sigma_{t}^{-1}}\left(\|\theta-\widehat \theta_{t_k}\|_{\widehat\Sigma_{t}}+\|\widehat \theta_{t_k}-\theta^*\|_{\widehat\Sigma_{t}}\right)\\
    &\leq 2\left\|\phi_{W_k}(s_t,a_t)\right\|_{\widehat \Sigma_{t}^{-1}}\left(\|\theta-\widehat \theta_{t_k}\|_{\widehat\Sigma_{t_k}}+\|\widehat \theta_{t_k}-\theta^*\|_{\widehat\Sigma_{t_k}}\right)\\
    &\leq4\widehat \beta_T\bar\sigma_t\left\|\bar\sigma_t^{-1}\phi_{W_k}(s_t,a_t)\right\|_{\widehat \Sigma_{t}^{-1}}
\end{align*}
where the first inequality is from the Cauchy-Schwarz inequality, the second one holds by \Cref{lem:determinant1}, and the third one follows from $\theta^*,\theta\in\mathcal{C}_k$ and $\widehat\beta_{t_k}\leq \widehat \beta_T$. We also know that 
$$\langle \phi_{W_k}(s_t,a_t),\theta-\theta^*\rangle=\mathbb{E}_{s'\sim\mathbb{P}_\theta(\cdot\mid s_t,a_t)}[W_k(s')]- \mathbb{E}_{s'\sim\mathbb{P}(\cdot\mid s_t,a_t)}[W_k(s')]\in[-H,H]$$
because $W_k(s)\in[0,H]$ for any $s\in\mathcal{S}$. Then it follows that
\begin{align*}
I_4&\leq \sum_{k=1}^{K_T}\sum_{t=t_k}^{t_{k+1}-1}\min\left\{H, 4\widehat \beta_T\bar\sigma_t\left\|\bar\sigma_t^{-1}\phi_{W_k}(s_t,a_t)\right\|_{\widehat \Sigma_{t}^{-1}}\right\}\\
&\leq 4\widehat\beta_T\sum_{k=1}^{K_T}\sum_{t=t_k}^{t_{k+1}-1}\bar\sigma_t\min\left\{1, \left\|\bar\sigma_t^{-1}\phi_{W_k}(s_t,a_t)\right\|_{\widehat \Sigma_{t}^{-1}}\right\}\\
&\leq 4\widehat \beta_T\underbrace{\sqrt{\sum_{t=1}^T\bar\sigma_t^2}}_{J_1}\underbrace{\sqrt{\sum_{k=1}^{K_T}\sum_{t=t_k}^{t_{k+1}-1}\min\left\{1, \left\|\bar\sigma_t^{-1}\phi_{W_k}(s_t,a_t)\right\|_{\widehat \Sigma_{t}^{-1}}^2\right\}}}_{J_2}
\end{align*}
where the second inequality holds because $H\leq 4\beta_T\bar\sigma_t$ and the third one is by the Cauchy-Schwarz inequality.

Note that by the Azuma-Hoeffding inequality (\Cref{lemma:azuma-hoeffding}), 
\begin{equation}\label{martingale2}\sum_{k=1}^{K_T}\sum_{t=t_k}^{t_{k+1}-1}\left([\mathbb{P}W_k^2](s_t,a_t)-W_k^2(s_{t+1})\right)\leq H^2\sqrt{2T\log(1/\delta)}
\end{equation}
holds with probability at least $1-\delta$.
By taking the union bound, all of the statement of \Cref{lem:confidence-ellipsoid}, the statement of \Cref{lem:regret-term3}, and~\eqref{martingale2} hold with probability at least $1-5\delta$. 

Now we suppose that all of the statement of \Cref{lem:confidence-ellipsoid}, the statement of \Cref{lem:regret-term3}, and~\eqref{martingale2} hold. For the term $J_1$, note that
\begin{align*}
  \sum_{t=1}^T\bar\sigma_t^2&= \sum_{k=1}^{K_T}  \sum_{t=t_k}^{t_{k+1}-1}\max\left\{H^2/d, [\bar{\mathbb{V}}_tW_k](s_t, a_t) + E_t\right\}\\
  &\leq \sum_{k=1}^{K_T}  \sum_{t=t_k}^{t_{k+1}-1}\max\left\{H^2/d, [{\mathbb{V}} W_k](s_t, a_t) + 2E_t\right\}\\
  &\leq TH^2/d + \sum_{k=1}^{K_T}  \sum_{t=t_k}^{t_{k+1}-1}[{\mathbb{V}} W_k](s_t, a_t)  + 2\sum_{t=1}^T E_t
\end{align*}
where the first inequality is from \Cref{lem:confidence-ellipsoid} and the second inequality holds because both $H^2/d$ and $[{\mathbb{V}} W_k](s_t, a_t) + 2E_t$ are nonnegative. Then we obtain from \Cref{lem:I4-variance,lem:I4-errors} that 
$$ \sum_{t=1}^T\bar\sigma_t^2 =\widetilde{\mathcal{O}}\left(H^2T/d + HT+H^2d\sqrt{T}+ d^{3/2}H^2\sqrt{T}\right)$$
where $\widetilde{\mathcal{O}}(\cdot)$ hides logarithmic factors in $TH/(\delta\lambda)$. Moreover, we know that $J_2\leq \sqrt{2d\log(1+T/\lambda)}$ by~\eqref{eq:vectorsum}.
Therefore, we finally deduce that
$$I_4 =\widetilde{\mathcal{O}}\left(\sqrt{d}\cdot \sqrt{H^2T/d + HT+H^2d\sqrt{T}+ d^{3/2}H^2\sqrt{T}}\cdot \sqrt{d}\right)= \widetilde{\mathcal{O}}\left(H\sqrt{dT} + d\sqrt{HT} + d^{7/4}HT^{1/4}\right)$$
where $\widetilde{\mathcal{O}}(\cdot)$ hides logarithmic factors in $TH/(\delta\lambda)$. 

Recall that when $\lambda = 1/B_\theta^2$ we have
\begin{align*}
    I_1&\leq d\sqrt{\mathrm{sp}(v^*)T} = d\sqrt{HT},\\
    I_2&= \widetilde{\mathcal{O}}\left(d\sqrt{HT}\right),\\
    I_3&\leq H\sqrt{2d\log(1/\delta)}.
\end{align*}
Lastly, setting $\gamma$ and $N$ as
$$\gamma = 1- \sqrt{\frac{d}{HT}}\quad\text{and}\quad N = \frac{1}{1-\gamma}\log\left(\frac{\sqrt{T}}{d\sqrt{H}}\right)=\sqrt{\frac{HT}{d}}\log\left(\frac{\sqrt{T}}{d\sqrt{H}}\right),$$ 
we have $$N\geq \frac{\log\left({\sqrt{T}}/{d\sqrt{H}}\right)}{\log(1/\gamma)},$$
in which case we get $T \gamma^N\leq d\sqrt{HT}$.

Therefore, we conclude that with probability at least $1-5\delta$
$$\regret(T) = \widetilde{\mathcal{O}}\left(H\sqrt{dT} + d\sqrt{HT} + d^{7/4}HT^{1/4}\right)$$
where $\widetilde{\mathcal{O}}(\cdot)$ hides logarithmic factors in $THB_\theta/\delta$.

\section{Experiments}

In this section, we empirically compare the performance of \texttt{UCLK-C} and \texttt{UCRL2-VTR}, which uses Bernstein-type exploration \citep{yuewu2022}. We run simulations on the MDP instance introduced in \Cref{sec:lb}. Given that both algorithms achieve minimax optimality for the MDP instance as it is communicating, our comparison provides an intuitive understanding of how controlling the span of the value function can improve regret performance. The next paragraph details the experimental setup and results.

\begin{figure}[h]
    \centering
    \includegraphics[width=0.5\linewidth]{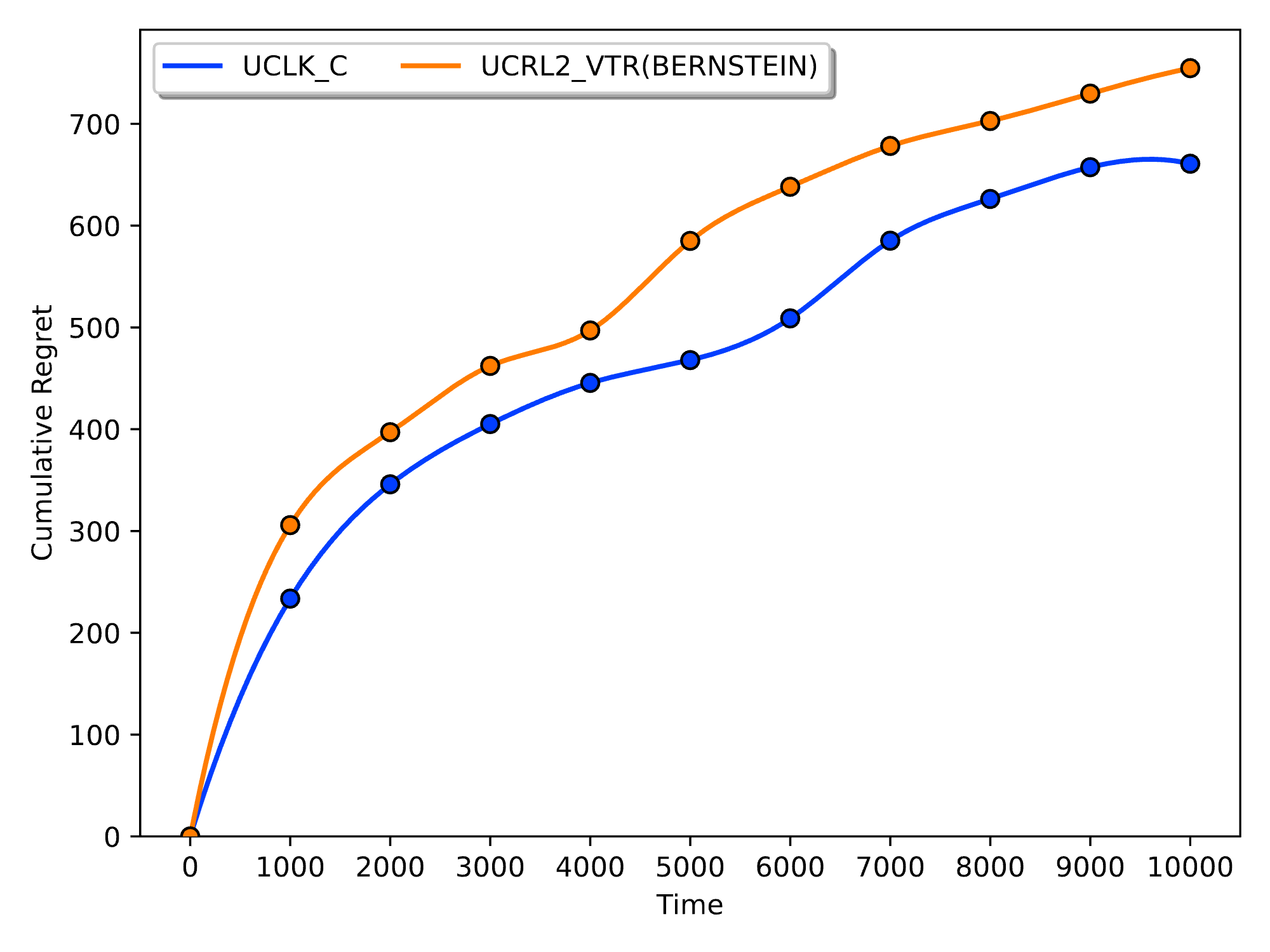}
    \caption{Regret comparison of \texttt{UCLK-C} and \texttt{UCRL2-VTR} (Bernstein-type), $\delta=1/120$ $(\Leftrightarrow D=120)$}
    \label{fig:regret_uclk-c_ucrl2-vtr_d_120}
\end{figure}

We choose $d=8$, thereby the resulting MDP consists of two states ($|\mathcal{S}|=2$) and 128 actions ($|\mathcal{A}|=2^{d-1}=128$). Furthermore, we set $\Delta = (d-1)/(15\sqrt{(2T\log2)/(5\delta)})$, scaling the original quantity introduced in \Cref{sec:lb} by a factor of $3$ in our implementation to ensure effective learning procedure. Specifically, using the original formula $(\Delta = (d-1)/(45\sqrt{(2T\log2)/(5\delta)}))$ results in a very small value, making the transition probabilities for the best and worst actions nearly identical, which hinders effective learning. One possible solution would be then to increase the value of $d$, but due to computing resource constraints, we instead scaled the constant factor. This adjustment does not affect the theoretical guarantee of the regret lower bound (see the proof of \citep[Theorem 5.5,][]{yuewu2022}) but allows for more practical implementation. %

Under this setting, we compared the two algorithms based on the average regret with respect to the best action, calculated over 10 realizations of the experiment conducted on a time horizon in units of thousands for each algorithm. The hyperparameters for each algorithm were properly tuned, and the averaged regret for the instance with $\delta = 1/120$ is plotted in \Cref{fig:regret_uclk-c_ucrl2-vtr_d_120}. Our proposed algorithm, \texttt{UCLK-C}, outperforms \texttt{UCRL2-VTR}, demonstrating the effectiveness of controlling the span for improving regret performance.

\end{document}